\documentclass[conference]{IEEEtran}
\usepackage{times}

\usepackage[numbers,sort,compress]{natbib}
\usepackage{multicol}
\usepackage[bookmarks=true]{hyperref}

\usepackage{comment}
\usepackage{amsthm}
\usepackage{mathtools}
\usepackage{amsmath}
\usepackage{amsfonts}
\usepackage{soul} %for strike-through text
\usepackage{amsmath,amsfonts,amssymb}
\usepackage[mathscr]{euscript}
\usepackage{hyperref}
\usepackage{graphicx}

\usepackage{algorithm}
\usepackage[noend]{algpseudocode}
\algrenewcommand{\algorithmicrequire}{\textbf{Input:}}
\algrenewcommand{\algorithmicensure}{\textbf{Output:}}

\hypersetup{
    colorlinks=true,
    linkcolor=blue,
    citecolor=blue,
    urlcolor=blue
}

\newtheorem{remark}{Remark}
\newtheorem{theorem}{Theorem}
\newtheorem{problem}{Problem}[section]
\newtheorem{lemma}{Lemma}
\newtheorem{definition}{Definition}[section]
\newcounter{subproblem}[problem] % Create a counter for subproblems
\usepackage{xcolor}
\newcommand{\cblue}[1]{\textcolor{blue}{#1}}
\newcommand{\cred}[1]{\textcolor{red}{#1}}
\renewcommand{\bf}[1]{\textbf{#1}}
\newcommand{\mycomment}[1]{}
\renewcommand{\t}[0]{^{\intercal}}
\newcommand{\inv}[0]{^{-1}}
\begin{document}

%\title{Template paper for the \\Robotics: Science and Systems Conference}
%\title{Composable Library of Macro-actions for Planning}
%\title{Verifiable Synthesis and Composition of a Library of Macro-actions for Probabilistic Planning under Control Constraints}
%\title{Verifiable Synthesis of a Composable Library of Macro-actions for Probabilistic Planning under Control Constraints}
%\title{Verifiable Synthesis of a Composable Library of Macro-actions for Probabilistic Planning}
%\title{Verifiable Synthesis of a Composable Roadmap for Computationally Tractable Long-horizon Probabilistic Planning}
%\title{Backward Reachable Trees for Probabilistic Planning under Control Constraints}
%\title{Synthesis of Verifiable Composable Trees for Probabilistic Planning under Control Constraints}
%\title{Synthesis of Composable Trees for Probabilistic Planning under Control Constraints: A Semidefinite Programming Approach}
%\title{Synthesis of Maximum Coverage Trees for Probabilistic Planning under Control Constraints: A Semidefinite Programming Approach}
\title{SDP Synthesis of Maximum Coverage Trees for Probabilistic Planning under Control Constraints}

% You will get a Paper-ID when submitting a pdf file to the conference system
%\author{Author Names Omitted for Anonymous Review. Paper-ID [add your ID here]}
\author{Naman Aggarwal, and Jonathan P. How \\ 
Laboratory for Information and Decision Systems, Massachusetts Institute of Technology}

%\author{\authorblockN{Michael Shell}
%\authorblockA{School of Electrical and\\Computer Engineering\\
%Georgia Institute of Technology\\
%Atlanta, Georgia 30332--0250\\
%Email: mshell@ece.gatech.edu}
%\and
%\authorblockA{Twentieth Century Fox\\
%\and
%\authorblockN{James Kirk\\ and Montgomery Scott}
%\authorblockA{Starfleet Academy\\
%San Francisco, California 96678-2391\\

% avoiding spaces at the end of the author lines is not a problem with
% conference papers because we don't use \thanks or \IEEEmembership

%Homer Simpson\authorrefmark{2},
%Montgomery Scott\authorrefmark{3} and
%Eldon Tyrell\authorrefmark{4}}
%Telephone: (800) 555--1212, Fax: (888) 555--1212}
%\authorblockA{\authorrefmark{4}Tyrell Inc., 123 Replicant Street, Los Angeles, California 90210--4321}}

\maketitle

\begin{abstract}
    The paper presents Maximal Covariance Backward Reachable Trees (MAXCOVAR BRT), which is a multi-query algorithm for planning of dynamic systems under stochastic motion uncertainty and constraints on the control input with explicit coverage guarantees. In contrast to existing roadmap-based probabilistic planning methods that sample belief nodes randomly and draw edges between them \cite{csbrm_tro2024}, under control constraints, the reachability of belief nodes needs to be explicitly established and is determined by checking the feasibility of a non-convex program. Moreover, there is no explicit consideration of coverage of the roadmap while adding nodes and edges during the construction procedure for the existing methods. Our contribution is a novel optimization formulation to add nodes and construct the corresponding edge controllers such that the generated roadmap results in provably maximal coverage under control constraints as compared to any other method of adding nodes and edges. We characterize formally the notion of coverage of a roadmap in this stochastic domain via introduction of the h-$\operatorname{BRS}$ (Backward Reachable Set of Distributions) of a tree of distributions under control constraints, and also support our method with extensive simulations on a 6 DoF model.
\end{abstract}
\IEEEpeerreviewmaketitle
%\section{Introduction}
\section{Introduction}

Multi-query motion planning entails the design of plans that can be reused across different initial configurations of the system. This is typically done via the offline construction of a roadmap of feasible trajectories in the state space, such that in real-time, for a pair of initial and goal configurations of the system, planning proceeds by connecting the initial and goal configurations to the roadmap followed by graph search to find a path \cite{prm1996,eirm}. The pre-computation of a roadmap is beneficial since it avoids the computational burden of finding plans from scratch for new configurations, and reuses search effort across queries. Such an approach of constructing a roadmap of feasible trajectories facilitates robust planning since a plan to meet the mission requirements could be computed quickly if the system finds itself in unexpected configurations due to unmodeled dynamics or disturbances during online operation. An important design consideration for any roadmap based planning algorithm is the coverage of the roadmap -- it is desirable to be able to re-use search effort across as many queries as possible, therefore, reasoning explicitly about the coverage of the roadmap becomes an important algorithmic design consideration.
%Multi-query planners aim to maximize re-use of search effort across multiple queries and this becomes an algorithmic consideration while constructing the \textit{roadmap}.

There is extensive work on motion planning via the construction of reusable roadmaps for deterministic systems \cite{tedrake2010lqr,majumdar2011multi,majumdar2017funnel,pipx}. The work proposed in \cite{tedrake2010lqr}, for instance, builds a tree of \textit{funnels}, which are regions of finite-time invariance around trajectories with associated feedback controllers, backwards from the goal in a space-filling manner. Paths to the goal could be found from different regions of the state space by traversing the branches of the constructed tree. Ref.~\cite{majumdar2017funnel} takes a different approach by populating a library with a finite set of feasible trajectories with associated funnels, such that the online planning proceeds through the sequential composition of the available funnels in the library. %The work proposed in \cite{pipx} focuses on feedback motion re-planning in dynamic environments by repairing graph-wiring online and via incremental graph search through a volumetric funnel graph.

\begin{figure}[t]
\vspace*{-0.15in}
    \centering
    \includegraphics[width=0.5\textwidth]{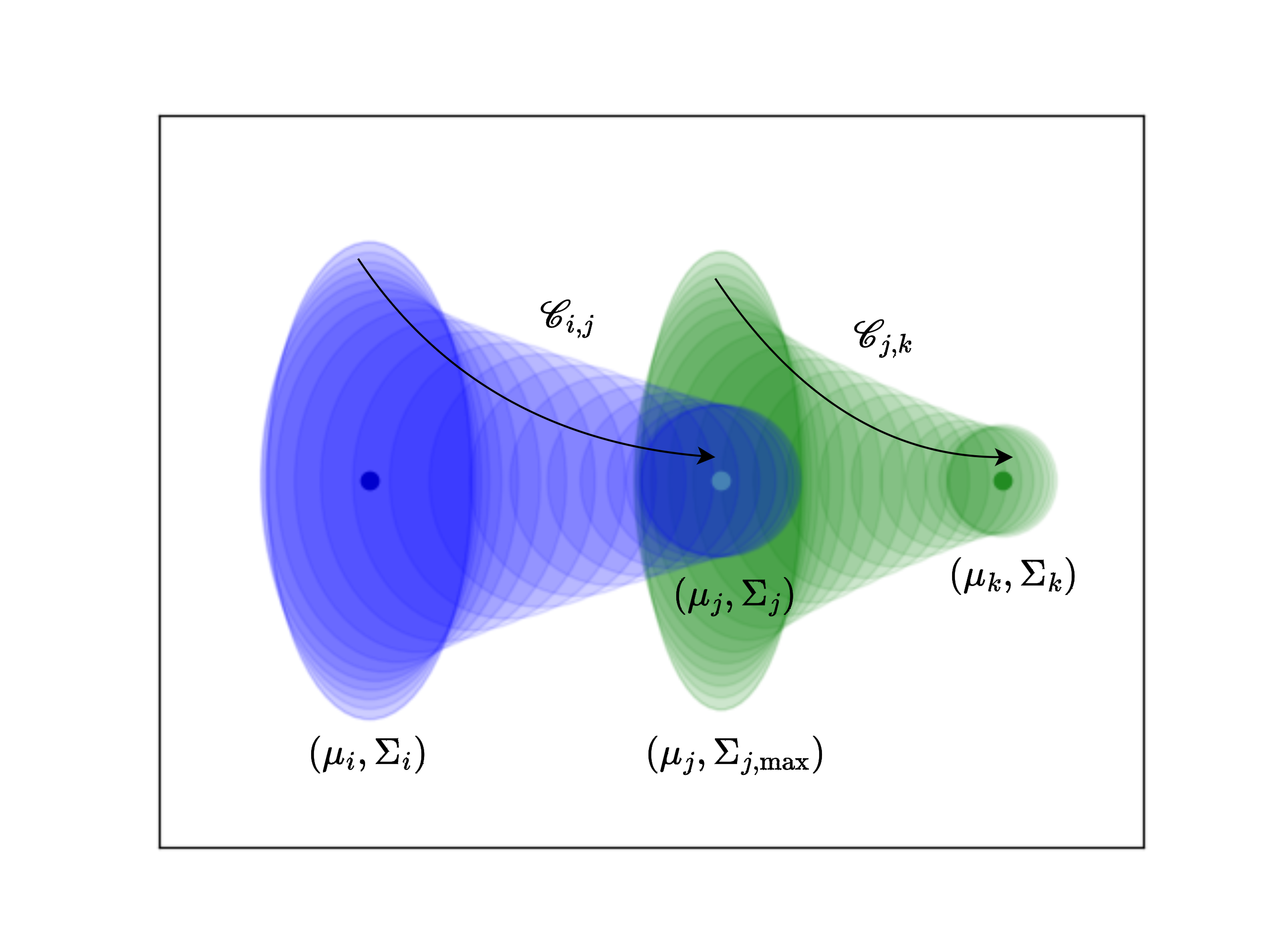}
\vspace*{-0.5in}
    \caption{An illustration of the recursive satisfaction of constraints (goal distribution reaching constraint, and chance constraints on the state and control input) for the overall trajectory initialized at the $(\mu_{i}, \Sigma_{i})$ distribution through sequential composition of two $N$-step controllers, $ \mathscr{C}_{i,j}$ (designed for the $i \rightarrow j$ maneuver) and $\mathscr{C}_{j,k}$ (designed for the $j_{\mathrm{max}} \rightarrow k$ maneuver).}
    %\label{fig:tree}
\end{figure}

In this paper, we are concerned with planning for systems with stochastic dynamics. A stochastic framework helps by making informed plans considering the probability of constraint violation and leads to a lesser conservative approach as compared to a deterministic framework considering worst-case uncertainty in the dynamics or enforcing hard constraints. CC-RRT \cite{luders2010chance} is a popular framework for single-query motion planning in stochastic systems. The algorithm proceeds by growing a forwards search tree of distributions over the state space, such that the obstacle avoidance constraints are satisfied probabilistically along each edge. The tree expansion phase involves randomly sampling open-loop control signals and simulating the system dynamics forward until the chance constraints on obstacle avoidance are satisfied, and adding the forward simulated state distribution to the tree as a node with the associated open-loop control sequence as the edge. A major limitation of constructing a search tree via sampling open-loop control signals is that the the second order moment of the state distributions (covariance) cannot be controlled via open-loop control even for the simplest case of linear-Gaussian systems. CC-RRT comes with a probabilistic completeness guarantee such that the algorithm finds paths to all goal distributions for which open-loop paths would exist \cite{luders2014thesis}. However, the set of all possible distributions that could be reached with open-loop control is a subset of the set of all distributions that could be reached with a richer control scheme such as feedback control. Particularly, for a mission with tight goal reaching constraints (a specific size of the goal covariance as required by the mission), CC-RRT might fail to find a path even though a closed-loop path might exist.

While CC-RRT is a single-query planner, there exist other roadmap based methods for planning in stochastic domains. FIRM \cite{agha2014firm} is a planning framework that builds a roadmap in the belief space by sampling stationary belief nodes and drawing belief stabilizing controllers as edges between the nodes. The closest in approach to our work, however, is the recently proposed CS-BRM \cite{csbrm-tro} that provides a faster planning scheme as compared to FIRM by allowing sampling of non-stationary belief nodes and adding finite-time feedback controllers as edges. CS-BRM is built on top of a series of latest work on covariance steering that allows finite-time satisfaction of terminal covariance constraints \cite{okamoto2018optimal,liu2022optimal,rapakoulias2023discrete}.
CS-BRM, however, does not assume control input constraints, and adds edges between randomly sampled belief nodes. In the absence of control constraints, the distribution steering problem decouples into the mean steering and covariance steering problem: both of which are tractable to solve \cite{rapakoulias2023discrete}. In the presence of control input constraints however, the reachability of belief nodes needs to be explicitly established. Even for the simplest case of a controllable linear-Gaussian system, it is no longer possible to drive arbitrary distributions to arbitrary distributions, and the existence of a steering maneuver for a pair of belief nodes must be determined by checking the feasibility of a nonconvex program. Therefore, CS-BRM is not efficient for settings with constraints on the control input due to unnecessary sampling of belief nodes that would not be added to the roadmap due to the non-existence of a steering control. 

Moreover, even for successful insertion of belief nodes, there is no explicit consideration of coverage while selecting nodes during the roadmap construction procedure, neither a post-analysis on the coverage characteristics of the roadmap generated by CS-BRM. We provide a systematic procedure that reasons about the coverage of the to-be-generated roadmap while adding nodes through a novel optimization formulation. The contributions of the paper are as follows:
\begin{itemize}
    \item  We characterize the notion of coverage of a roadmap in the stochastic domain formally via what we call h-$\operatorname{BRS}$ (Backward Reachable Set of Distributions) of the tree, which is the set of all distributions that can reach the goal in h-hops where 1 hop is a finite-time horizon of N-steps under control constraints.
    \item We propose a novel method of edge construction such that the BRT constructed following our method finds paths from provably the largest set of initial distributions for the overall planning problem, i.e.\ provides maximal coverage.
    \item We supplement our approach with extensive simulations on a 6 DoF model.
\end{itemize}
%Our statement of contribution: explicit characterization of coverage of the roadmap in the form of h-\operatorname{BRS} which is the set of all distributions that can reach the goal in h-hops where 1 hop is a finite-time horizon of N-steps under control constraints.
%Things to talk about: their method does not reason about control constraints. And also, there is no formal consideration of coverage of the roadmap while constructing the roadmap. Drawback: do not reason about control constraints and reachability of belief nodes thereof. Neither do they reason explicitly about coverage while constructing the roadmap.

%\section{Related Work}
%\input{related_work.tex}
\section{Problem Statement}\label{sec:problem_statement}
Consider a discrete-time stochastic linear system represented by the following difference equation
\begin{equation}\label{eq:dynamics}
    \mathbf{x}_{t+1} = A \mathbf{x}_{t} + B \mathbf{u}_{t} + D \mathbf{w}_{t}
\end{equation}
where $\mathbf{x}_{t} \in \mathcal{X}$, $\mathbf{u}_{t} \in \mathcal{U}$ are the state and control inputs at time $t$, respectively, $\mathcal{X} \subseteq \mathbb{R}^{n}$ is the state space, $\mathcal{U} \subseteq \mathbb{R}^{m}$ is the control space, $n, m \in \mathbb{N}$, $D \in \mathbb{R}^{n \times n}$ and $\mathbf{w}_{t}$ represents the stochastic disturbance at time $t$ that is assumed to have a Gaussian distribution with zero mean and unitary covariance. We also assume that $\{w_{t}\}$ is an \textit{i.i.d.} process, and that $A$ is non-singular. Note that we use the notation $\mathcal{N}(\mu, \Sigma)$ to denote a Gaussian distribution with mean $\mu$ and covariance $\Sigma$.

We define a finite-horizon optimal control problem, $\operatorname{OPT-STEER}(\mathcal{I}, \mathcal{G}, N)$, where $\mathcal{I}$, $\mathcal{G}$, $N$ are the initial distribution, goal distribution, and the time-horizon corresponding to the control problem respectively. $\operatorname{OPT-STEER}(\mathcal{I}, \mathcal{G}, N)$ solves for a control sequence that is optimal with respect to a performance index $J$, and that steers the system from an initial distribution $\mathcal{I}$ to a goal distribution $\mathcal{G}$ in a time-horizon of length $N$ while avoiding obstacles such that the control sequence respects the prescribed constraints on the control input at each time-step:
\begin{equation}
    %\min _{\mathbf{u}_k} \quad J=\mathbb{E}\left[\sum_{k=0}^{N-1} \mathbf{x}_k^{\top} Q_k \mathbf{x}_k + \mathbf{u}_k^{\top} R_k \mathbf{u}_k\right] \tag{OPT-STEER} \label{eq:J}
    \min _{ \Phi_{k} } \quad J=\mathbb{E}\left[\sum_{k=0}^{N-1} \mathbf{x}_k^{\top} Q_k \mathbf{x}_k + \mathbf{u}_k^{\top} R_k \mathbf{u}_k\right] \tag{OPT-STEER} \label{eq:J}
\end{equation}
such that, for all $k = 0, 1, \cdots N - 1$,
\begin{subequations}
\begin{align}
    &\mathbf{x}_{k+1} = A\mathbf{x}_{k} + B\mathbf{u}_{k} + D \mathbf{w}_{k}, \label{eq:opt_dynamics} \\
    &\mathbf{x}_{0} \sim \mathcal{N}(\mu_{\mathcal{I}}, \Sigma_{\mathcal{I}}), \label{eq:opt_init_dist} \\
    &\mathbf{x}_{N} \sim \mathcal{N}(\mu_{N}, \Sigma_{N}), \label{eq:opt_final_dist} \\
    &\mu_{N} = \mu_{\mathcal{G}}, \label{eq:opt_goal_mean} \\
    &\Sigma_{N} \preceq \Sigma_{\mathcal{G}}, \label{eq:opt_goal_covar} \\
    &\mathbb{P}(\mathbf{x}_{k} \in \mathcal{X}) \geq 1- \epsilon_{x}, \label{eq:opt_state_constraint} \\
    &\mathbb{P}(\mathbf{u}_{k} \in \mathcal{U}) \geq 1- \epsilon_{u}, \label{eq:opt_control_constraint}
\end{align}
\end{subequations}
where $\Phi_{k}(.)$ is the parameterization for the finite-horizon controller such that $ \mathbf{u}_{k} \coloneqq \Phi_{k}(\mathbf{x}_{k}) $.
\begin{comment}
We assume the existence of a finite-horizon controller that can achieve the desired steering maneuver from the initial distribution $\mathcal{I}$ to the goal distribution $\mathcal{G}$ in $N$ time steps. In other words, the system of equations (\ref{eq:opt_dynamics})-(\ref{eq:opt_control_constraint}) is feasible and represents a non-empty set in the space of all $N$-step controllers. We discuss the feasibility of the any such covariance steering maneuver in detail in Section~\ref{sec:prelim}.
\end{comment}

Under constraints on the control input of the form (\ref{eq:opt_control_constraint}), not all \textit{instances} of $\operatorname{OPT-STEER}(\mathcal{I}, \mathcal{G}, N)$ have a solution where an \textit{instance} is described by the tuple $(\mathcal{I}, N)$ of the initial distribution and the planning horizon for a fixed goal distribution $\mathcal{G}$. In particular, $ \operatorname{OPT-STEER}(\mathcal{I}, \mathcal{G}, N) $ has a solution if and only if the constraint set of the optimization problem defined by the set of equations (\ref{eq:opt_dynamics})-(\ref{eq:opt_control_constraint}) is non-empty in the space of the decision variables $\{\Phi_{k}\}_{k=0}^{N-1}$. %\st{, or that the instance $\operatorname{OPT-STEER}(\mathcal{I}, \mathcal{G}, N)$ is feasible}.
Alternatively, under control constraints, the existence of an $N$-step controller that steers the system from an initial distribution to a goal distribution needs to be established explicitly by solving a feasibility question even when the underlying system $(A, B)$ is controllable. This existence (or not) of a control sequence for a planning instance defined by the triplet $\mathcal{I}, \mathcal{G}, N$ is a special artifact of the presence of control input constraints in distribution steering (see Section~\ref{sec:cs_review}). This paper presents solutions to the following problem:
\begin{comment}
This happens because under constraints (\ref{eq:opt_dynamics})-(\ref{eq:opt_control_constraint}), the mean and covariance dynamics are coupled, and not all distributions can be driven to arbitrary distributions even if the underlying system is controllable. The control budget at each time-step is divided between driving the mean, and in controlling the size of the covariance, which in planning problems with tight goal reaching constraints becomes crucial.
\end{comment}
\begin{problem}\label{problem:problem_eng}
    Find paths to the goal distribution $\mathcal{G}$ from all query initial distributions $\mathcal{I}$ for which paths (respecting all the constraints) exist.
\end{problem}
\begin{comment}
\noindent Equivalently,
\begin{problem}\label{problem:problem_math}
    Solve $\operatorname{OPT-STEER}(., \mathcal{G}, .)$ in a multi-query sense, i.e., find solutions for all possible tuples of the query initial distribution and the planning horizon $(\mathcal{I}, N)$ for which $\operatorname{OPT-STEER}(\mathcal{I}, \mathcal{G}, N)$ is feasible.
\end{problem}
\end{comment}
Note that we use the terms \textit{path} and \textit{control sequence} interchangeably in this paper. Problem~\ref{problem:problem_eng} is an instance of multi-query planning, and is a difficult problem to solve in it's full generality.
We employ graph-based methods to build a roadmap in the space of distributions of the state of the system such that the roadmap could be used across query initial distributions.

%We take a two-pronged approach: for any query initial distribution, first, a procedure that verifies if there exists and also returns a feasible path that reaches the goal distribution under all constraints; second, an approximate optimal path is found via further optimization of the discovered feasible path.

%The closest in approach to our paper is this recent work on constructing belief space roadmaps for multi-query planning in partially observable domains \cite{csbrm_tro2024}. A major difference from our setting however is the absence of constraints on the control input. The method proposed in \cite{csbrm_tro2024} draws belief nodes randomly and adds edges between them. As we show in Section~\ref{sec:cs_review}, even in the simplest case of a Linear-Gaussian system, under constraints on the control input, the existence of a steering control for a problem instance is determined by checking the feasibility of a non-convex problem. In other words, distributions could not be steered to any distribution for a given time-horizon, and hence the method proposed in \cite{csbrm_tro2024} is not applicable.
We propose a novel method of edge construction such that the edge controller can be re-used across a wider set of initial distributions for the edge, and such that the BRT constructed following the novel method of edge construction finds paths from provably the largest set of initial distributions for the overall planning task. We supplement our approach with extensive simulations.

We begin with a discussion on the feasibility of a covariance steering problem instance in Section~\ref{sec:cs_review}, followed by a review of an existing tractable method for covariance steering in Section~\ref{sec:convex_relax}. We discuss our approach in Section~\ref{sec:approach} followed by a discussion on experiments in Section~\ref{sec:experiments}.
\begin{comment}
    The roadmap is generated in such a manner that feasible paths from maximum possible query initial distributions are found compared to any other method of tree generation (construction procedure discussed at length in following sections).
\end{comment}
%\section{Approach}
%\section{Mathematical Preliminaries}\label{sec:prelim}
\section{Finite Horizon Covariance Steering under Control Constraints}\label{sec:prelim}
In this section, we discuss the finite horizon covariance steering problem, and the feasibility of any covariance steering problem instance.
\subsection{Feasibility}\label{sec:cs_review}
The development in this subsection follows closely \cite{liu2022optimal} \cite{rapakoulias2023discrete}. We consider a linear state feedback parameterization for the controller to solve (\ref{eq:J}) as follows,
\begin{equation}\label{eq:control_param}
    \mathbf{u}_{k} = K_{k}(\mathbf{x}_{k} - \mu_{k}) + v_{k}
\end{equation}
where $ K_{k} \in \mathbb{R}^{n \times m} $ is the feedback gain that controls the covariance dynamics, and $ v_{k} \in \mathbb{R}^{n}$ is the feedforward gain that controls the mean dynamics. Under a linearly parameterized controller as described in (\ref{eq:control_param}), the state distribution remains Gaussian at all times and we can express the objective $J$ (\ref{eq:J}) completely in terms of the first and second order moments of the state process: 
$$
J \!\! = \!\! \sum_{k=0}^{N-1} \operatorname{tr}({Q_{k}\Sigma_{k}}) + \operatorname{tr}(R_{k} K_{k} \Sigma_{k} K^{\intercal}_{k} ) + \mu^{\intercal}_{k} Q_{k} \mu_{k} + v_{k}^{\intercal} R_{k} v_{k}, \label{eq:J_simp}
$$
We consider polytopic state and control constraints of the form $ \mathcal{X} \coloneqq \{ x_{k} \in \mathbb{R}^{n} \ \vert \ \alpha^{\intercal}_{x} x_{k} \leq \beta_{x} \} $, $ \mathcal{U} \coloneqq \{ u_{k} \in \mathbb{R}^{m} \ \vert \ \alpha^{\intercal}_{u} u_{k} \leq \beta_{u} \} $ such that,
\begin{align}
    \mathbb{P}( \alpha^{\intercal}_{x} \mathbf{x}_{k} \leq \beta_{x} ) &\geq 1 - \epsilon_{x} \label{eq:state_chance_constraint}, \\
    \mathbb{P}( \alpha^{\intercal}_{u} \mathbf{u}_{k} \leq \beta_{u} ) &\geq 1 - \epsilon_{u}, \label{eq:control_chance_constraint}
\end{align}
where $ \alpha_{x} \in \mathbb{R}^{n} $, $ \alpha_{u} \in \mathbb{R}^{m} $, and $\beta_{x}, \beta_{u} \in \mathbb{R}$. $\epsilon_{x}, \epsilon_{u} \in \left[0, 0.5\right]$ represent the tolerance levels with respect to state and control constraint violation respectively, and $ \{\alpha_{x} \mathbf{x}_{k}\}^{N-1}_{k=0} $ and $ \{\alpha_{u} \mathbf{u}_{k}\}^{N-1}_{k=0} $ are univariate random variables with 
first and second order moments,
\begin{align}
&\mathbb{E}( \alpha_{x} \mathbf{x}_{k} ) = \alpha_{x} \mu_{k} \label{eq:state_constraint_first_moment} \\
&\mathbb{E}( \alpha_{u} \mathbf{u}_{k} ) = \alpha_{u} v_{k} \label{eq:control_constraint_first_moment} \\
&\mathbb{E}( \alpha^{\intercal}_{x} ( \mathbf{x}_{k} - \mu_{k} ) ( \mathbf{x}_{k} - \mu_{k} )^{\intercal} \alpha_{x}) = \alpha^{\intercal}_{x} \Sigma_{k} \alpha_{x} \label{eq:state_constraint_second_moment} \\
% &\mathbb{E}( \alpha^{\intercal}_{u} ( \mathbf{u}_{k} - v_{k} ) ( \mathbf{u}_{k} - v_{k} )^{\intercal} \alpha_{u}) = \alpha\t_{u} U_{k} \Sigma\inv_{k} U\t_{k} \alpha_{u} \label{eq:control_constraint_second_moment}.
 &\mathbb{E}( \alpha^{\intercal}_{u} ( \mathbf{u}_{k} - v_{k} ) ( \mathbf{u}_{k} - v_{k} )^{\intercal} \alpha_{u}) = \alpha\t_{u} K_{k} \Sigma_{k} K\t_{k} \alpha_{u} \label{eq:control_constraint_second_moment}.
\end{align}
Ref.~\cite{okamoto2018optimal} shows that the chance constraints can be written as,
\begin{subequations}
\begin{align}
    &\Phi^{-1}(1-\epsilon_{x}) \sqrt{ \alpha^{\intercal}_{x} \Sigma_{k} \alpha_{x} } + \alpha^{\intercal}_{x} \mu_{k} - \beta_{x} \leq 0, \label{eq:state_soft_simplified} \\
    %&\Phi^{-1}(1-\epsilon_{u}) \sqrt{ \alpha\t_{u} U_{k} \Sigma\inv_{k} U\t_{k} \alpha_{u} } + \alpha^{\intercal}_{u} v_{k} - \beta_{u} \leq 0, \label{eq:control_soft_simplified}
    &\Phi^{-1}(1-\epsilon_{u}) \sqrt{ \alpha\t_{u} K_{k} \Sigma_{k} K\t_{k} \alpha_{u} } + \alpha^{\intercal}_{u} v_{k} - \beta_{u} \leq 0, \label{eq:control_soft_simplified}
\end{align}
\end{subequations}
where $ \Phi^{-1}(.)$ is the inverse cumulative distribution function of the normal distribution. Therefore, the optimization problem \ref{eq:J} can be recast as the  nonlinear program,
\begin{align}
    \min_{ \Sigma_{k}, K_{k}, \mu_{k}, v_{k} } J = \sum_{k=0}^{N-1} &\operatorname{tr}({Q_{k}\Sigma_{k}}) + \operatorname{tr}(R_{k} K_{k} \Sigma_{k} K^{\intercal}_{k} ) \nonumber \\ &+ \mu^{\intercal}_{k} Q_{k} \mu_{k} + v_{k}^{\intercal} R_{k} v_{k}, \label{eq:J_simp}
\end{align}
such that for all $k = 0, 1, \cdots, N - 1$,
\begin{subequations}
\begin{align}
    &\begin{aligned}
        \Sigma_{k+1} = A &\Sigma_{k} A^{\intercal} + B K_{k}\Sigma_{k} A\t + A\Sigma_{k} K\t_{k} B\t \nonumber \\ &+ B K_{k} \Sigma_{k} K\t_{k} B\t + DD\t,
    \end{aligned} \tag{\ref{eq:J_simp}a} \label{eq:covar_prop_nlp} \\
    &\Sigma_{0} = \Sigma_{\mathcal{I}}, \tag{\ref{eq:J_simp}b} \label{eq:covar_init_nlp} \\
    &\Sigma_{N} \preceq \Sigma_{\mathcal{G}}, \tag{\ref{eq:J_simp}c} \label{eq:covar_goal_nlp} \\
    &\mu_{k+1} = A \mu_{k} + B v_{k}, \tag{\ref{eq:J_simp}d} \label{eq:mean_prop_nlp} \\
    &\mu_{0} = \mu_{\mathcal{I}}, \tag{\ref{eq:J_simp}e} \label{eq:mean_init_nlp} \\
    &\mu_{N} = \mu_{\mathcal{G}}\tag{\ref{eq:J_simp}f} \label{eq:mean_goal_nlp}, \\
    &\Phi^{-1}(1-\epsilon_{x}) \sqrt{ \alpha^{\intercal}_{x} \Sigma_{k} \alpha_{x} } + \alpha^{\intercal}_{x} \mu_{k} - \beta_{x} \leq 0, \tag{\ref{eq:J_simp}g} \label{eq:state_constraint_nlp} \\
    &\Phi^{-1}(1-\epsilon_{u}) \sqrt{ \alpha\t_{u} K_{k} \Sigma_{k} K\t_{k} \alpha_{u} } + \alpha^{\intercal}_{u} v_{k} - \beta_{u} \leq 0. \tag{\ref{eq:J_simp}h} \label{eq:control_constraint_nlp}
\end{align}
\end{subequations}
The existence of an $N$-step steering control that drives the state distribution from $\mathcal{I}$ to $\mathcal{G}$ satisfying all the constraints is determined by the feasibility of the set of equations (\ref{eq:covar_prop_nlp})-(\ref{eq:control_constraint_nlp}) which represents a non-convex set in the space of the decision variables $ \{K_{k}\}_{k=0}^{N-1}, \{v_{k}\}_{k=0}^{N-1} $. Determining if a non-convex set is empty or not from it's algebraic description is an NP-hard problem in general, and the complexity of this feasibility check scales with the time-horizon since the problem size becomes larger and the number of variables increase. Problem~\ref{problem:problem_eng} is concerned with finding feasible paths from all possible initial distributions, and our solution methodology proceeds by building a backward reachable tree of feasible paths from the goal distribution, and sequencing them together to find a feasible path at run-time for the query initial distribution, hence avoiding solving for the feasible path of the query distribution from scratch (see Section~\ref{sec:approach} for details).
%%%%%%%%%%%%%%%%%%%%%%%%%%%%%%%%%%%%%%%%%%%%%%%%%%%%%%
\subsection{Convex Relaxation} \label{sec:convex_relax}
Following the development in \cite{liu2022optimal} \cite{rapakoulias2023discrete}, making the change of variables $ U_{k} = K_{k} \Sigma_{k} $ and introducing an auxiliary variable $ Y_{k} $, we can relax (\ref{eq:J_simp}) to a convex semidefinite program as,
\begin{align}
    \min_{ \Sigma_{k}, U_{k}, Y_{k}, \mu_{k}, v_{k} } J = \sum_{k=0}^{N-1} &\operatorname{tr}({Q_{k}\Sigma_{k}}) + \operatorname{tr}(R_{k} Y_{k}) \nonumber \\ &+ \mu^{\intercal}_{k} Q_{k} \mu_{k} + v_{k}^{\intercal} R_{k} v_{k}, \label{eq:J_relaxed}
\end{align}
such that for all $k = 0, 1, \cdots, N - 1$,
\begin{subequations}
\begin{align}
    C_{k} &\triangleq U_{k} \Sigma_{k}\inv U\t_{k} - Y_{k} \preceq 0 \tag{\ref{eq:J_relaxed}a} \label{eq:Ck}, \\
    G_{k} &\triangleq A\Sigma_{k}A\t + B U_{k} A\t + A U\t_{k} B\t \nonumber \\ &~~+ B Y_{k} B\t + D D\t - \Sigma_{k+1} = 0, \tag{\ref{eq:J_relaxed}b} \label{eq:Gk} \\
    &\text{subject to (\ref{eq:covar_init_nlp}) - (\ref{eq:mean_goal_nlp})}, \tag{\ref{eq:J_relaxed}c}
\end{align}
\end{subequations}
where the constraint (\ref{eq:Ck}) can be expressed as an LMI using the Schur complement as follows,
$\left[\begin{array}{ll}\Sigma_k & U_k^{\top} \\ U_k & Y_k\end{array}\right] \succeq 0.$
From (\ref{eq:Ck}), (\ref{eq:control_soft_simplified}) is further relaxed to,
\begin{align}
    \Phi\inv(1-\epsilon_{u}) \sqrt{ \alpha\t_{u} Y_{k} \alpha_{u} } + \alpha^{\intercal}_{u} v_{k} - \beta_{u} \leq 0. \label{eq:control_soft_relaxed}
\end{align}
Due to the presence of the square root, neither of (\ref{eq:state_soft_simplified}) and (\ref{eq:control_soft_relaxed}) are convex. Ref.~\cite{rapakoulias2023discrete} proposes a linearization of (\ref{eq:state_soft_simplified}) and (\ref{eq:control_soft_relaxed}), and since the square root is a concave function, the tangent line serves as the global overestimator,
\begin{equation}
    \sqrt{x} \leq \frac{1}{2 \sqrt{x_0}} x+\frac{\sqrt{x_0}}{2}, \quad \forall x, x_0>0
\end{equation}
Therefore, the constraints (\ref{eq:state_soft_simplified}) and (\ref{eq:control_soft_simplified}) are finally approximated as,
\begin{equation}
\begin{aligned}
&\begin{aligned}
\Phi^{-1}\left(1-\epsilon_x\right) & \frac{1}{2 \sqrt{\alpha_x^{\top} \Sigma_r \alpha_x}} \alpha_x^{\top} \Sigma_k \alpha_x+\alpha_x^{\top} \mu_k \\
& -\left(\beta_x-\Phi^{-1}\left(1-\epsilon_x\right) \frac{1}{2} \sqrt{\alpha_x^{\top} \Sigma_r \alpha_x}\right) \leq 0
\end{aligned}\\
&\begin{aligned}
\Phi^{-1}\left(1-\epsilon_u\right) & \frac{1}{2 \sqrt{\alpha_u^{\top} Y_r \alpha_u}} \alpha_u^{\top} Y_k \alpha_u+\alpha_u^{\top} v_k \\
& -\left(\beta_u-\Phi^{-1}\left(1-\epsilon_u\right) \frac{1}{2} \sqrt{\alpha_u^{\top} Y_r \alpha_u}\right) \leq 0
\end{aligned}
\end{aligned}
\end{equation}
where $\Sigma_{r}, Y_{r}$ are some reference values. The linearized constraints now form a convex set.

The finite horizon constrained covariance steering problem (\ref{eq:J}) is finally approximated as the following convex semidefinite program,
\begin{align}
    \min_{ \Sigma_{k}, U_{k}, Y_{k}, \mu_{k}, v_{k} } J \label{eq:J_final}
\end{align}
such that for all $ k = 0, 1, \cdots, N-1$,
\begin{subequations}
\begin{align}
    &\mu_{k+1} = A \mu_{k} + B v_{k}, \tag{\ref{eq:J_final}a} \label{eq:1} \\
    &C_{k}(\Sigma_{k}, U_{k}, Y_{k}) \preceq 0, \tag{\ref{eq:J_final}b} \label{eq:2} \\
    &G_{k} (\Sigma_{k+1}, \Sigma_{k}, Y_{k}, U_{k}) = 0, \tag{\ref{eq:J_final}c} \label{eq:3} \\
    &\ell\t\Sigma_{k}\ell + \alpha\t_{x}\mu_{k} - \overline{\beta}_{x} \leq 0, \tag{\ref{eq:J_final}d} \label{eq:4} \\
    &e\t Y_{k} e + \alpha\t_{u}v_{k} - \overline{\beta}_{u} \leq 0, \tag{\ref{eq:J_final}e} \label{eq:5}
\end{align}
\end{subequations}
where $J$ is defined in (\ref{eq:J_relaxed}).

\section{MAXCOVAR BRT: A MAXIMUM COVERAGE TREE FOR PROBABILISTIC PLANNING}\label{sec:approach}
\begin{comment}
\subsection{Preliminaries: Finite Horizon Covariance Steering under Constraints}
    As will be discussed in a later section, there does not necessarily exist a sequence of controls that can achieve the prescribed distribution steering maneuver under control constraints even if the pair (A,B) is controllable.
\subsection{Maximal Covariance Controller Design}
\subsection{Construction of the Backward Reachable Tree of Distributions}
\end{comment}

We solve Problem~\ref{problem:problem_eng} by constructing a Backward Reachable Tree (BRT) of distributions that verifiably reach the goal distribution $\mathcal{G}$ under constraints on the control input. As discussed previously, in the presence of control constraints, existence of a control sequence that steers the system from an initial distribution to a goal distribution is established by solving a feasibility problem. The size of this feasibility check scales with the time-horizon and the BRT enables a faster feasibility check on a long time-horizon by checking the feasibility of reaching any existing node on the BRT instead of directly checking feasibility against the goal distribution.

We refer to this idea as \textit{recursive feasibility} since the branches of the tree can be thought of as carrying a certificate of feasibility along its edges from the goal in a backwards fashion s.t. guaranteeing feasibility to any of the children nodes in the tree guarantees feasibility to all the upstream parent nodes and consequently the root node that corresponds to the goal distribution.
%A Backward Reachable Tree $\mathcal{T}$ corresponding to a goal distribution $\mathcal{G}$ is a tree of distributions on the state space represented through a set of nodes, and a set of edge controllers such that all constraints are satisfied for the system evolution along each edge.
%We hypothesize that as number of samples goes to infinity and the tree becomes more dense, the tree would find paths to all possible initial distributions from which paths to the goal exist, as is demonstrated empirically in Section~\ref{sec:experiments}. We leave a formal proof of probabilistic-completeness as future work, and just discuss the coverage property of a finite-sample tree for now.

We introduce a novel objective function MAX-COVAR, as discussed in the following subsection, for adding nodes and constructing edge controllers such that the resulting tree provides maximum \textit{coverage}. We also characterize formally the notion of coverage mathematically in this section.
\subsection{MAX-COVAR: Novel Objective for Construction of the Edge Controller} \label{sec:max_covar}
We define a procedure to construct an $N$-step edge controller as follows. The procedure $\operatorname{MAXCOVAR}$ takes in as input a candidate initial mean $\mu_{\mathrm{initial}}$, a target distribution at the end of the $N$-step steering maneuver $ ( \mu_{\mathrm{target}}, \Sigma_{\mathrm{target}} ) $ and computes an initial covariance $ \Sigma_{\mathrm{initial}} $ in a \textit{maximal} sense, henceforth referred to as $\Sigma_{\mathrm{initial, max}}$, and the associated control sequence $\mathscr{C} \coloneqq \{ \mathscr{C}_{k} \}_{k=0}^{N-1}$ that achieves the corresponding steering maneuver. The control at time $k$ is a tuple of the feedback and the feedforward term s.t.\ $\mathscr{C}_{k} = ( K_{k}, v_{k} )$.
\begin{align}
    \min_{ \Sigma_{k}, K_{k}, \mu_{k}, v_{k} } J_{\mathrm{MAX\textunderscore COVAR}} = -\lambda_{\mathrm{min}}(\Sigma_{0}) \tag{MAX-COVAR}\label{eq:J_max_covar}
\end{align}
such that for all $k = 0, 1, \cdots, N - 1$,
\begin{subequations}
\begin{align}
    &\begin{aligned}
        \Sigma_{k+1} = A &\Sigma_{k} A^{\intercal} + B K_{k}\Sigma_{k} A\t + A\Sigma_{k} K\t_{k} B\t \nonumber \\ &+ B K_{k} \Sigma_{k} K\t_{k} B\t + DD\t,
    \end{aligned} \\ %\label{eq:covar_prop_maxcovar_error} \\
    %&\Sigma_{N} \preceq \Sigma_{\mathcal{G}}, \tag{\ref{eq:J_max_covar}b} \label{eq:covar_goal_maxcovar} \\
    &\lambda_{\mathrm{max}}(\Sigma_{N}) \leq \lambda_{\mathrm{min}}(\Sigma_{\mathcal{G}}), \label{eq:covar_goal_maxcovar} \\
    &\mu_{k+1} = A \mu_{k} + B v_{k}, \label{eq:mean_prop_maxcovar} \\
    &\mu_{0} = \mu_{i},\ \ \mu_{N} = \mu_{\mathcal{G}} \label{eq:mean_maxcovar} \\
    %&\mu_{N} = \mu_{\mathcal{G}} \label{eq:N_mean_goal_maxcovar}, \\
    &\Phi^{-1}(1-\epsilon_{x}) \sqrt{ \alpha^{\intercal}_{x} \Sigma_{k} \alpha_{x} } + \alpha^{\intercal}_{x} \mu_{k} - \beta_{x} \leq 0, \label{eq:state_constraint_maxcovar} \\
    &\Phi^{-1}(1-\epsilon_{u}) \sqrt{ \alpha\t_{u} K_{k} \Sigma_{k} K\t_{k} \alpha_{u} } + \alpha^{\intercal}_{u} v_{k} - \beta_{u} \leq 0, \label{eq:control_constraint_maxcovar}
\end{align}
\end{subequations}
where $ \lambda_{\mathrm{min}}(.) $ is the minimum eigenvalue operator. $\lambda_{\mathrm{min}}(A)$ is a concave function of the positive semidefinite matrix variable $A$, therefore \ref{eq:J_max_covar} is a convex minimization objective. The feasible region of the above optimization problem is non-convex and we use the similar lossless convexification as described in Section~\ref{sec:convex_relax} to tackle \ref{eq:J_max_covar}.
\begin{comment}
    \begin{theorem}
    The change of variables, and the convex relaxation proposed in Section~\ref{sec:convex_relax} optimally recovers the solution to \ref{eq:J_max_covar} with the proposed novel objective function.
\end{theorem}
\begin{proof}
    To be proved.
\end{proof}
\end{comment}

\textit{\textbf{Notation}}: We define a predicate $\operatorname{FEASIBLE} \coloneqq \operatorname{FEASIBLE}(q, p, N)$ that returns a boolean TRUE or FALSE if there exists a feasible $N$-step control sequence such that the system of equations (\ref{eq:covar_goal_maxcovar})-(\ref{eq:control_constraint_maxcovar}) defined for $(\mu_{\mathcal{I}}, \Sigma_{\mathcal{I}}) = (\mu_{q}, \Sigma_{q})$, and $(\mu_{\mathcal{G}}, \Sigma_{\mathcal{G}}) = (\mu_{p}, \Sigma_{p})$ is feasible. Also, we use the notation $ q \xrightarrow[N]{\mathscr{C}} p$ to denote that the mean and covariance dynamics initialized at $(\mu_{q}, \Sigma_{q})$ and driven by the N-step control sequence $\mathscr{C}$ satisfy the state and control chance constraints (\ref{eq:state_constraint_maxcovar})-(\ref{eq:control_constraint_maxcovar}) at all time-steps and the terminal goal reaching constraint (\ref{eq:covar_goal_maxcovar}) corresponding to $(\mu_{\mathcal{G}}, \Sigma_{\mathcal{G}}) = (\mu_{p}, \Sigma_{p})$.

\begin{remark}\label{remark:reuse}
(Reuse) Let $\mathscr{C}$ be a $N$-step control sequence s.t.\ $ \mathcal{I} \xrightarrow[N]{\mathscr{C}} \mathcal{G} $, then it follows that $ \mathcal{I}^{-} \xrightarrow[N]{\mathscr{C}} \mathcal{G} $ for all $ \mathcal{I}^{-}$ s.t.\  $\mu_{\mathcal{I}^{-}} = \mu_{\mathcal{I}}$, and $\Sigma_{\mathcal{I}^{-}} \preceq \Sigma_{\mathcal{I}}$. In other words, a control sequence computed for the steering maneuver from $\mathcal{I}$ to $\mathcal{G}$ remains a feasible maneuver from $\mathcal{I}^{-}$ to $\mathcal{G}$ and thus could be reused.
\end{remark}
The above directly follows from an analysis of equations (\ref{eq:covar_goal_maxcovar})-(\ref{eq:control_constraint_maxcovar}) and we omit a formal proof.

\textbf{\textit{Rationale behind}} $\operatorname{MAX-COVAR}$: \ref{eq:J_max_covar} is based on the maximization of the minimum eigenvalue of $\Sigma_{\mathrm{initial}}$ for a given initial mean, and a desired target distribution. It is based on the observation that if $ \mathcal{I} \xrightarrow{\mathscr{C}} \mathcal{G} $, then $ \mathcal{I}' \xrightarrow{\mathscr{C}} \mathcal{G} \ \ \forall \ \ \mathcal{I}' \text{ s.t. } \mu_{\mathcal{I}} = \mu_{\mathcal{I}'} \text{ and } \Sigma_{\mathcal{I}'} \prec \Sigma_{\mathcal{I}} $. In other words, if the system trajectory initialized at $ \mathcal{I} $ respects all the constraints and reaches the target distribution $\mathcal{G}$ under control sequence $ \mathscr{C} $, then for the system initialized at $ \mathcal{I}' $ such that $ \mu_{\mathcal{I}'} = \mu_{\mathcal{I}}$ and $ \Sigma_{\mathcal{I}'} \prec \mathcal{I} $, the system satisfies all the constraints under the same control sequence $\mathscr{C}$. 

Therefore, we aim to find a $ (\Sigma_{\mathrm{initial}}, \mathscr{C}) $ such that the computed $\mathscr{C}$ could be reused across largest possible number of initial distributions, i.e., find $ \mathcal{I} $ such that $ \{ \mathcal{I}' \ \vert \ \mathcal{I}' \prec \mathcal{I}  \} $ is the largest. This leads to the maximization of the minimum eigenvalue of the initial covariance as a natural objective function for our search.

\textbf{\textit{Significance of}} MAX-COVAR: $\operatorname{MAX-COVAR}$ provides a certificate of reachability for any goal distribution in terms of the maximum permissible value of the minimum eigenvalue of the covariance at any query mean for which there exists a feasible control sequence that can achieve the corresponding steering maneuver under control constraints. For instance, consider a goal distribution $ (\mu_{\mathcal{G}}, \Sigma_{\mathcal{G}}) $ and a query mean $\mu_{q}$, and let $ \Sigma_{q,\mathrm{max}}$, $ \mathscr{C}_{q, \mathrm{max}} $ be such that,
\begin{align*}
    \Sigma_{q, \mathrm{max}}, \mathscr{C}_{q, \mathrm{max}} \longleftarrow \operatorname{MAXCOVAR}( \mu_{q}, (\mu_{\mathcal{G}, \Sigma_{\mathcal{G}}}), N).
\end{align*}
It follows from the above that $\forall \ \Sigma \succ 0$ s.t.\ $ \lambda_{\mathrm{min}}(\Sigma) > \lambda_{\mathrm{min}}(\Sigma_{q,\mathrm{max}})$, $ \nexists \ \mathscr{C}$ s.t.\ $ (\mu_{q}, \Sigma) \xrightarrow[N]{\mathscr{C}} \mathcal{G}$ by definition of MAX-COVAR otherwise $ \lambda_{\mathrm{min}} (\Sigma_{q, \mathrm{max}}) $ is not the maximum possible minimum eigenvalue of the initial covariance for the existence of a feasible path and we arrive at a contradiction. On the other hand, all matrices $ \Sigma \succ 0 $ such that their minimum eigenvalue is less than the minimum eigenvalue of the covariance matrix computed in the maximal sense i.e.\ $ \lambda_{\mathrm{min}}(\Sigma) < \lambda_{\mathrm{min}}(\Sigma_{q, \mathrm{max}}) $ have lesser \textit{coverage} than the maximal covariance matrix, a fact that is formalized later in Lemma~\ref{lemma:sigma_brs_lemma}. %In summary, .....

\subsection{Construction of the MAXCOVAR BRT}\label{sec:construction_brt}
The algorithm proceeds by building a tree $\mathcal{T(\mathcal{G})}$ represented through a set of nodes $\{\mathcal{\nu}_{i}\}$, and a set of edge controllers $\{\varepsilon_{i,j}\}$. Each node $i$, $ \nu_{i} $, is a tuple $ \left( \mu_{i}, \Sigma_{i}, p_{i}, \mathscr{C}_{i}, \mathrm{ch}_{i} \right) $ where $ \mu_{i} , \Sigma_{i}$ are the mean and the covariance of the distribution stored in the node, $p_{i}$ is a pointer to the parent node, $ \mathscr{C}_{i} \coloneqq \{ K^{i, p_{i}}_{t}, v^{i, p_{i}}_{t} \}_{t=0}^{N-1} $ is the $N$-step control sequence stored at the node that steers the state distribution from $ (\mu_{i}, \Sigma_{i}) $ to $ (\mu_{p_{i}}, \Sigma_{p_{i}}) $, and $ \mathrm{ch}_{i} $ is the list of pointers of all the children node of node $i$ in the tree $\mathcal{T}(\mathcal{G})$.

$ \varepsilon $ is another data structure that stores all the edge information for the tree $\mathcal{T}(\mathcal{G})$, such that, $ \varepsilon_{i,j} \coloneqq \{ K^{i,j}_{t}, v^{i,j}_{t} \}_{t=0}^{N-1} $ stores the $N$-step control sequence that steers the state distribution from node $i$ to node $j$ if such an edge exists, and is empty otherwise.

\begin{algorithm}[t]
\caption{Constructing the MAXCOVAR BRT}
\label{alg:construct_brt}
\begin{algorithmic}[1]
    \Require $\mathcal{G}$, $N$, $n_\mathrm{iter}$ %Input
    \Ensure  $\mathcal{T}$ %Output
    \State $\nu \gets \phi$, $\varepsilon \gets \phi$
    \State $ \nu_{0} \gets \operatorname{CREATENODE}(\mu_{\mathcal{G}}, \Sigma_{\mathcal{G}}, \operatorname{NONE}, \operatorname{NONE}, \{\})$
    \State $\nu \gets \nu \cup \{\nu_{0}\}$
    \For{$i \gets 1 \textrm{ to } n_\mathrm{iter}$}
        \State %$\mathrm{idx}_\textrm{sampled}, \nu_{\textrm{sampled}} \gets \operatorname{RAND}(\nu)$
        $\nu_{k} \gets \operatorname{RAND}(\nu)$
        \State $ \mu_{q} \gets \operatorname{RANDMEANAROUND}(\nu_{k}, r_{\textrm{sample}}) $
        \State $ \textrm{status}, \Sigma_{\mathrm{max}}, \mathscr{C}_{q} \gets \operatorname{MAXCOVAR}(\mu_{q}, (\mu_{k}, \Sigma_{k} ), N)$
        \If{$\textrm{status} \neq \textrm{infeasible}$}
            \State $ \textrm{idx} \gets \operatorname{size}(\mathcal{\nu}) + 1 $
            \State $ \mathrm{ch}_{\mathrm{idx}} \gets \{\}$
            \State $ \nu_{\mathrm{new}} \gets \operatorname{CREATENODE}(\mu_{q}, \Sigma_{\mathrm{max}}, \textrm{idx}, k, \mathrm{ch}_{\mathrm{idx}})$
            \State $ \mathrm{ch}_{k} \gets \mathrm{ch}_{k} \bigcup \{ \textrm{idx}\} $
            \State $ \varepsilon_{\textrm{idx},k} \gets \mathscr{C}_{q} $
            \State $ \varepsilon \gets \varepsilon \bigcup \{ \varepsilon_{\textrm{idx},k} \} $
        \EndIf
    \EndFor
    \State $  \mathcal{T} \gets \{ \nu, \varepsilon \} $
    \State $ \textrm{\textbf{Return}} \ \  \mathcal{T}$
\end{algorithmic}
\end{algorithm}
Now, we discuss the essential sub-routines of the above algorithm.
\subsubsection{Node selection}
The tree is grown in the spirit of finding paths from all query initial distributions for which paths would exist to the goal distribution. In our implementation, the nodes are selected randomly according to the Voronoi bias (of the first order moment of the nodes) to bias population of the $1-$BRS of the node distributions whose corresponding Voronoi regions are relatively unexplored in the sense of the first order moment of the distributions.
%Other node sampling heuristics and mechanisms could also be used. The main take-away of our approach is the recursive verification of distributions to the goal through verification of distributions to sub-goals (children distributions of the goal distribution $\mathcal{G}$ in $\mathcal{T}(\mathcal{G}))$ and the resulting savings in computation thereof.
\begin{comment}
    Let's consider an example to demonstrate the efficacy of our approach. Consider distribution $p$ such that $p \in k_{1}-\operatorname{BRS}(p_{i})$ for some $k_{1} > 0$ such that $p_{i} \in k_{0}-\operatorname{BRS}(\mathcal{G})$ for some $k_{0} > 0$, then $ p \in (k_{1} + k_{0})-\operatorname{BRS}(\mathcal{G})$. Let $p$ be such that $ p \notin k'-\operatorname{BRS} $ for $k' < k_{1} + k_{0}$, i.e.\ a path to the goal distribution $\mathcal{G}$ shorter than $ (k_{1} + k_{0})N $ does not exist from $p$. Therefore, to find a feasible path to the goal, we need to solve $\operatorname{FEASIBLE}(p, \mathcal{G}, (k_{1} + k_{0})N)$, a problem of bigger size compared to solving $\operatorname{FEASIBLE}(p, p_{i}, k_{1}N)$ to find a feasible path to distribution $p_{i}$, and re-using the controllers from thereon to construct a feasible path to the goal $\mathcal{G}$.
\end{comment}

\subsubsection{Node expansion}
Once a node to expand has been selected, a query mean is sampled from a neighbourhood of some radius around it and a connection is attempted through the $\operatorname{MAXCOVAR}$ method for edge construction. Let's say the $k$th node on the tree containing the distribution $ (\mu^{(k)}, \Sigma^{(k)}) $ has been selected to expand, and let $ \mu_{q} $ be the query mean sampled from a neighbourhood around $ \mu^{(k)} $ through the $\operatorname{RANDMEANAROUND}(., r)$ module where $r$ is some sampling radius. We solve the following optimization to construct the edge,
\begin{align}
    \Sigma_{\mathrm{max}}, \mathscr{C}_{\mathrm{max}} \longleftarrow \operatorname{MAXCOVAR}( \mu_{q}, (\mu^{(k)}, \Sigma^{(k)}), N)
\end{align}
$ (\mu_{q}, \Sigma_{\mathrm{max}}) $ is added as a node to the tree with the edge controller $\mathscr{C}_{\mathrm{max}}$ and $( \mu^{(k)}, \Sigma^{(k)} )$ as the parent if the status of the above optimization problem (as returned by the solver) is not infeasible.

\textit{\textbf{Definitions}}: We now define the following mathematical objects that will aid the analysis and further discussion of our proposed approach. The $h$ hop backward reachable set of distributions for a distribution $p$, h-$\operatorname{BRS}(p)$ is defined as follows,
\begin{align}\label{eq:hbrs_def}
    h-\operatorname{BRS}(p) =  \{ (\mu_{q}, \Sigma_{q}) \vert \operatorname{FEASIBLE}(q, p, hN) = \operatorname{TRUE} \}.
\end{align}
We also define the $h$-$\operatorname{BRS}$ of a tree of distributions $\mathcal{T}$ as,
    \begin{align}
        \text{h-}\operatorname{BRS}(\mathcal{T}) = \bigcup_{i \in \nu(\mathcal{T})} \text{($h - d_{i}$)-}\operatorname{BRS}(\nu_{i}), \label{eq:tree_brs_def}
    \end{align}
    %where $ d_{i} $ is the distance of the $i$-th node from the goal node in terms of the number of hops.
    where $\nu(\mathcal{T})$ is the set of all vertices in the tree, and $d_{i} $ is the distance of the $i$-th node from the root node in terms of the number of hops.
%\cred{Mention that the h-BRS of the goal distribution is approximated through the h-BRS of the BRT grown from the goal distribution.}

\textbf{\textit{Concatenation of control sequences}}: A concatenation $\mathscr{C}_{A,B}$ of two control sequences $ \mathscr{C}_{A} $, $\mathscr{C}_{B}$ of lengths $N_{A}, N_{B}$ respectively is a control sequence of length $N_{A} + N_{B}$ represented through the $\bigcup$ operator i.e.\ $\mathscr{C}_{A,B} \coloneqq \mathscr{C}_{A} \bigcup \mathscr{C}_{B}$, s.t.\ $\mathscr{C}_{A,B}(t) \coloneqq C_{A}(t) \ \forall \ t = 0, 1, \cdots N_{A} - 1$, and $\mathscr{C}_{A,B}(t) \coloneqq C_{B}(t-N_{A}) \ \forall \ t = N_{A}, N_{A} + 1, \cdots N_{A} + N_{B} - 1$. Note that the concatenation operator $\bigcup$ is \textit{non-commutative} in the two argument control sequences, i.e.\ $ \mathscr{C}_{A} \bigcup \mathscr{C}_{B} \neq \mathscr{C}_{B} \bigcup \mathscr{C}_{A} $.

To establish / guarantee feasibility of the query distribution to the goal distribution, it is sufficient to guarantee feasibility to any distribution that is already verified to reach the goal. This idea can be seen in the following lemma on sequential composition of control sequences ensuring satisfaction of state and control chance constraints along the overall concatenated trajectory,
\begin{lemma}\label{lemma:sequential_composition}(Feasibility through Sequential Composition) Let $ (\mu_{i}, \Sigma_{i}), (\mu_{j}, \Sigma_{j}), (\mu_{k}, \Sigma_{k}), \mathscr{C}_{i,j}, \mathscr{C}_{j,k}, h_{i,j}, h_{j,k}$ be such that \ $ (\mu_{i}, \Sigma_{i}) \xrightarrow[h_{i,j} N]{\mathscr{C}_{i,j}} (\mu_{j}, \Sigma_{j}) $, and $ (\mu_{j}, \Sigma_{j}) \xrightarrow[h_{j, k} N]{\mathscr{C}_{j, k}} (\mu_{k}, \Sigma_{k}) $. It follows that $ (\mu_{i}, \Sigma_{i}) \xrightarrow[(h_{i,j} + h_{j, k}) N]{\mathscr{C}_{i,k}} (\mu_{k}, \Sigma_{k})$ where $ \mathscr{C}_{i,k} \coloneqq \mathscr{C}_{i,j} \bigcup \mathscr{C}_{j,k} $.
\end{lemma}
\begin{proof}
    We want to prove that the system initialized at the distribution $( \mu_{i}, \Sigma_{i} )$ and driven by the control sequence $ \mathscr{C}_{i,k} $ which is obtained from the concatenation of the two control sequences $ \mathscr{C}_{i,j} $, and $ \mathscr{C}_{j,k}$ i.e.\ $ \mathscr{C}_{i,k} = \mathscr{C}_{i,j} \bigcup \mathscr{C}_{j,k} $ satisfies all the state and control chance constraints of the form (\ref{eq:state_constraint_maxcovar})-(\ref{eq:control_constraint_maxcovar}) for a trajectory of length $ (h_{i,j} + h_{j,k}) N $, and the terminal goal reaching constraints of the form (\ref{eq:mean_maxcovar}) and (\ref{eq:covar_goal_maxcovar}) for the goal distribution $( \mu_{k}, \Sigma_{k})$.
    
    The mean and covariance dynamics at any time $t$ as a function of the intial distribution and the control sequence are represented as $ \mu_{t}(\mu_{i}, \mathscr{C}_{i,k})$ and  $ \Sigma_{t}(\Sigma_{i}, \mathscr{C}_{i,k})$ respectively.
    
    For $t = 0, 1, \cdots h_{i,j} N - 1$, $ \mu_{t}(\mu_{i}, \mathscr{C}_{i,k}) = \mu_{t}(\mu_{i}, \mathscr{C}_{i,j}) $, and $ \Sigma_{t}(\Sigma_{i}, \mathscr{C}_{i,k}) = \Sigma_{t}(\Sigma_{i}, \mathscr{C}_{i,j}) $. Since $ \mathscr{C}_{i,j} $ is a feasible control sequence for the $(i,j)$ maneuver s.t.\ $ (\mu_{i}, \Sigma_{i}) \xrightarrow[N_{i,j}]{\mathscr{C}_{i,j}} (\mu_{j}, \Sigma_{j}) $, all the state and control chance constraints are satisfied by $ \mu_{t}( \mu_{i}, \mathscr{C}_{i,j}) $ and $\Sigma_{t}( \Sigma_{i}, \mathscr{C}_{i,j}) $ for $t = 0, 1, \cdots N_{i,j} - 1$, and $ \Sigma_{N_{i,j}} (\Sigma_{i}, \mathscr{C}_{i,k})$ is s.t.\ $ \lambda_{\mathrm{max}}( \Sigma_{N_{i,j}} (\Sigma_{i}, \mathscr{C}_{i,k}) ) \leq \lambda_{\mathrm{min}}( \Sigma_{j} )$ which implies $ \Sigma_{N_{i,j}} (\Sigma_{i}, \mathscr{C}_{i,k}) \preceq \Sigma_{j} $.

    The rest of the maneuver for $t = N_{i,j}, N_{i,j} + 1, \cdots N_{i,j} + N_{j,k} - 1$ could be thought of as an $N_{j,k}$ step maneuver initialized at $ \mu_{N_{i,j}} $ and $ \Sigma_{N_{i,j}} $. $ \mathscr{C}_{j,k} $ is s.t.\ $ (\mu_{j}, \Sigma_{j}) \xrightarrow[N_{j, k}]{\mathscr{C}_{j, k}} (\mu_{k}, \Sigma_{k}) $, therefore from Remark~\ref{remark:reuse}, $ (\mu_{j}, \Sigma_{N_{i,j}} (\Sigma_{i}, \mathscr{C}_{i,k}) ) \xrightarrow[N_{j, k}]{\mathscr{C}_{j, k}} (\mu_{k}, \Sigma_{k}) $ since $\Sigma_{N_{i,j}} (\Sigma_{i}, \mathscr{C}_{i,k}) \preceq \Sigma_{j}$.
\end{proof}
\subsection{Planning through the BRT} \label{sec:planning}
In this section, we discuss our approach to find feasible paths to the goal through a BRT.

\textit{\textbf{Finding a feasible path}}: To find a feasible path to the goal for a query distribution $q \coloneqq (\mu_{q}, \Sigma_{q})$, single hop connections are attempted one-by-one to $M$ nearest nodes on the BRT for some hyperparameter $M$. For a candidate node $\nu_{k}$ on the BRT, the following problem is solved,
\begin{align}
    \mathscr{C}_{q} \longleftarrow \operatorname{OPT-STEER}( (\mu_{q}, \Sigma_{q}), (\mu^{(k)}, \Sigma^{(k)}), N).
\end{align}
\begin{comment}
    Now argue that the sequence of controls obtained by concatenating C_q with the pre-computed controllers also satisfies all the constraints.
\end{comment}
The search for a feasible path terminates once a connection has successfully been established to one of the existing nodes on the BRT, and is given by a concatenation of the above computed control sequence $\mathscr{C}_{q}$ with the pre-computed controllers stored in the sequence of edges of the tree from the $k$th node to the root node. Let $ \nu_{k} $ be a distance of $d_{k}$ hops away from the goal s.t.\ $ \mathrm{idx}_{0}, \mathrm{idx}_{1}, \cdots, \mathrm{idx}_{d_{k}}$ be the sequence of nodes encountered from the $k$th node to the root node where $\mathrm{idx}_{0} = k$ and \ $ \mathrm{idx}_{d_{k}} = 0 $. Thererfore, the feasible path from $q$ to the goal $\mathcal{G}$ is obtained as, $ \mathscr{C}_{q, \mathcal{G}} = \mathscr{C}_{q} \bigcup \mathscr{C}_{ \mathrm{idx}_{0}, \mathrm{idx}_{1} } \bigcup \mathscr{C}_{ \mathrm{idx}_{1}, \mathrm{idx}_{2} } \cdots \bigcup \mathscr{C}_{ \mathrm{idx}_{d_{k} - 1}, \mathrm{idx}_{d_{k}} } $ s.t.\ $ q \xrightarrow[(d_{k} + 1)N]{\mathscr{C}_{q, \mathcal{G}}} \mathcal{G}$, which follows from Lemma~\ref{lemma:sequential_composition}.%We refer to this idea as \textit{recursive feasibility} since the branches of the tree can be thought of as carrying a certificate of feasibility along its edges from the goal in a backwards fashion s.t. guaranteeing feasibility to any of the children nodes in the tree guarantees feasibility to all the upstream parent nodes and hence the root node, and the feasible control sequence is obtained by the concatenation of control sequences stored along the edges with the real-time computed control sequence ($\mathscr{C}_{q}$ above).
%Essentially, to find a feasible path to the goal, it is sufficient to find a path to any of the existing nodes on the tree and re-use the stored controllers at the edges to complete the steering maneuver from the query distribution to the goal distribution. We refer to this property as \textit{recursive feasibility} since to guarantee feasibility of the query distribution, it is sufficient to guarantee feasibility to any node that is already verified to be feasible to the goal distribution under control constraints (a node is called feasible if there exists a control sequence that can steer the system from the node distribution to the goal distribution under all constraints). 

\textit{\textbf{Implication of recursive feasibility on the speed-up in computing a feasible path}}: Say that the query distribution $q$ is such that $q \in \text{t}-\operatorname{BRS}(\mathcal{G})$ and $q \notin \text{t'}-\operatorname{BRS}(\mathcal{G}) \ \forall \ t' < t$, i.e.\ a path from $q$ to the goal shorter than $t$-hops does not exist. Therefore, to compute a feasible path that steers the system from $q$ to $\mathcal{G}$ without reusing any of the pre-computed controllers from the BRT, a feasibility instance $\operatorname{FEASIBLE}(q, \mathcal{G}, tN)$ of size $tN$ needs to be solved. Alternatively, if the search for a feasible path is carried through attempting connections to the BRT, the expenditure on compute is that of solving a feasibility instance of size $N$ a maximum $M$ number of times resulting in an order of magnitude savings in computation. More details about the empirical experiments and results can be found in Section~\ref{sec:experiments}.
%Add the complexity of SDP comments. Backup would be to say it has been empirically observed how the complexity scales; find the discussion in the Experiments sub-section.

\textit{\textbf{Maximum Coverage of the} MAXCOVAR BRT}: As discussed earlier, re-use of controllers stored along the edges of a BRT can lead to significant speedup in the computation of a feasible path to the goal. Therefore, it is a desirable property for a BRT to be such that paths from as large a number of query initial distributions as possible can be found to the tree. This property is characterized in terms of the \textit{coverage} of the BRT, and for any given BRT $\mathcal{T}$, its coverage is quantified through the set of all distributions that can reach the tree in $h$ hops i.e.
\begin{align}\label{eq:cover_tree_def}
    \operatorname{Cover}(\mathcal{T}) \coloneqq \text{h}-\operatorname{BRS}(\mathcal{T}).
\end{align}
We now show that the BRT constructed through the novel objective function for edge construction defined in Section~\ref{sec:max_covar} henceforth referred to as MAXCOVAR BRT discovers feasible paths from provably the largest possible set of query initial distributions compared to any other procedure of edge construction. This is formalized in Theorem~\ref{theorem:coverage} below. We proceed by first proving Lemma~\ref{lemma:sigma_brs_lemma} that talks about the coverage of a single node and is used as a building block for Theorem~\ref{theorem:coverage} that concerns the coverage of a tree (multiple nodes).

\begin{lemma}\label{lemma:sigma_brs_lemma}
    For $\Sigma_{\mathrm{max}}, \Sigma^{-}_{\mathrm{max}} \succ 0$, s.t.\ $\lambda_{\mathrm{min}}( \Sigma_{\mathrm{max}} ) > \lambda_{\mathrm{min}}( \Sigma^{-}_{\mathrm{max}} )$, $ h\text{-}\operatorname{BRS}( \mu, \Sigma_{\mathrm{max}} ) \supseteq h\text{-}\operatorname{BRS}( \mu, \Sigma^{-}_{\mathrm{max}} ) \ \forall \ h \geq 1 $ for all planning scenes. Also, there exist planning scenes $ \{ (\alpha_{x}, \beta_{x}, \epsilon_{x}), (\alpha_{u}, \beta_{u}, \epsilon_{u}) \} $ s.t.\ $ h\text{-}\operatorname{BRS}( \mu, \Sigma_{\mathrm{max}} ) \supset h\text{-}\operatorname{BRS}( \mu, \Sigma^{-}_{\mathrm{max}} ) \ \forall \ h \geq 1$.
\end{lemma}
\begin{proof}
    Proof relegated to Appendix~\ref{sec:appendix}.
\end{proof}

\begin{theorem}[Maximum Coverage]\label{theorem:coverage}
    $\text{h-}\operatorname{BRS}(\mathcal{T}^{(r)}_{\mathrm{MAXCOVAR}} ) \supseteq \text{h-}\operatorname{BRS}(\mathcal{T}^{(r)}_{\mathrm{ANY}} )$ for all planning scenes, and there always exist planning scenes such that $\text{h-}\operatorname{BRS}(\mathcal{T}^{(r)}_{\mathrm{MAXCOVAR}} ) \supset \text{h-}\operatorname{BRS}(\mathcal{T}^{(r)}_{\mathrm{ANY}} ) \ \forall \ r \geq 1, \ \forall \ h \geq 1$.
\end{theorem}
\begin{proof}
    Proof relegated to Appendix~\ref{sec:appendix}.
\end{proof}

\section{Experiments}\label{sec:experiments}
To illustrate our method, we conduct experiments for the motion planning of a quadrotor in a 2D plane. The lateral and longitudinal dynamics of the quadrotor are modeled as a triple integrator leading to a 6 DoF model with state matrices,
$$
A=\left[\begin{array}{ccc}
I_2 & \Delta T I_2 & 0_2 \\
0_2 & I_2 & \Delta T I_2 \\
0_2 & 0_2 & I_2
\end{array}\right], ~B=\left[\begin{array}{c}
0_2 \\
0_2 \\
\Delta T I_2
\end{array}\right], ~ D=0.1 I_6,
$$
a time step of 0.1 seconds, a horizon of $N = 20$, and a goal distribution $\mathcal{G}$ for the planning task as follows:
\begin{align*}
    \mu_{\mathcal{G}} = \mathbf{0}_{6 \times 1}, \quad \Sigma_{\mathcal{G}} = 0.1 *  \mathbf{I}_{6\times 6}.
\end{align*}
The control input space is characterized by a bounding box represented as 
$\alpha_u=\left\{\left[
\pm 1,0 \right]^{\top}, 
\left[0,\pm 1 \right]^{\top}\right\}$, 
$\beta_u=\{ \pm 25, \pm 25\}$.
The chance constraint linearization is performed around $ \Sigma_{r} = 1.2 \ \mathbf{I}_{6\times6} $, and $ Y_{r} = 15 \ \mathbf{I}_{2\times2} $. All the optimization programs are solved in Python using cvxpy \cite{cvxpy2016}.

\begin{figure}[t]
    \centering
    \includegraphics[trim=0 50 0 50, clip, width=\columnwidth]{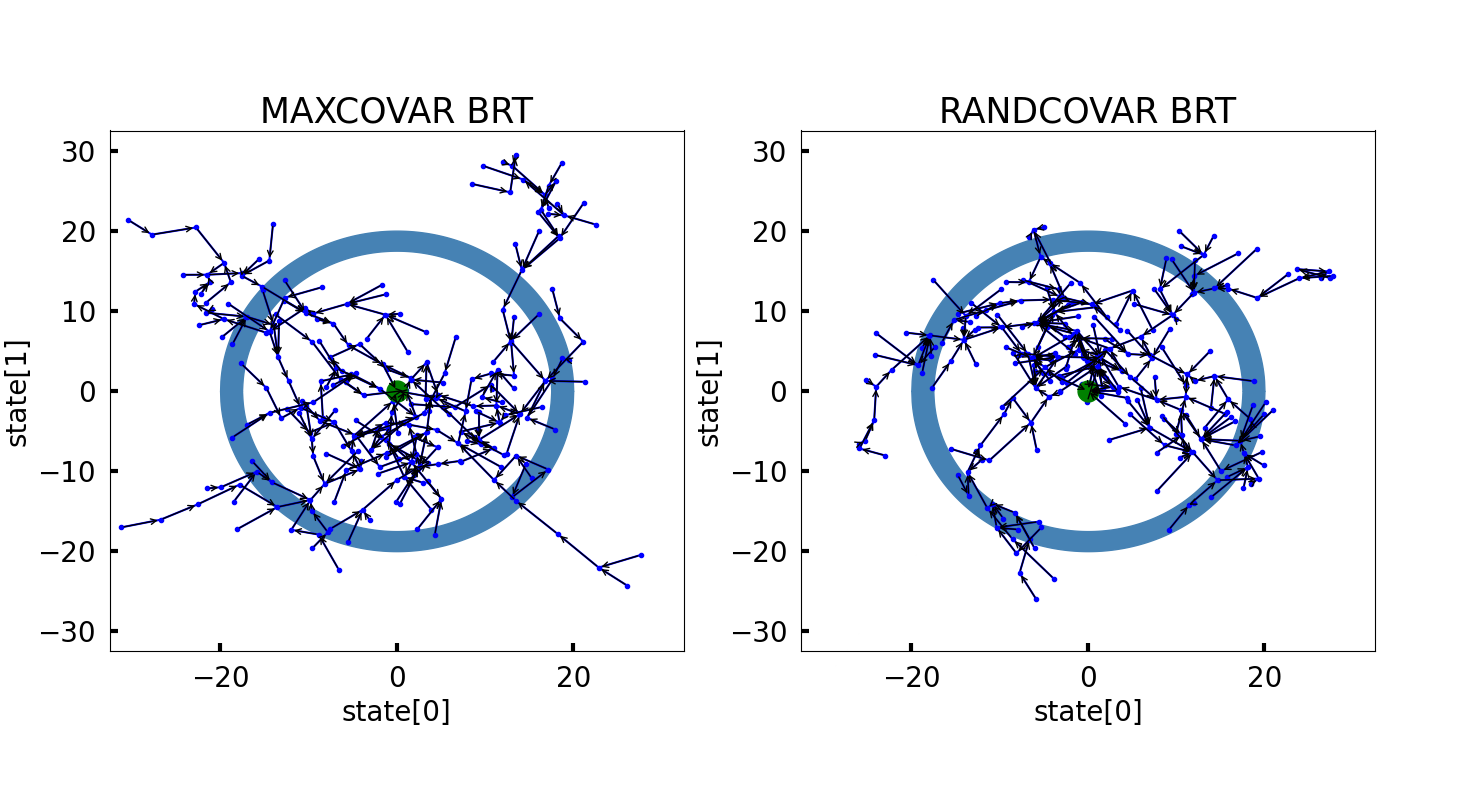}
    \vspace*{-.25in} 
    \caption{Coverage experiment setup: MAXCOVAR and RANDCOVAR BRTs}
    \label{fig:coverage_exp_setup}
\end{figure}

\textit{\textbf{Construction of the BRTs}}: For the tree construction procedure, a sampling radius of $r_{\mathrm{sample}} = \left[ \pm 5, \pm 5, \pm 2.5, \pm 2.5, \pm 1.25, \pm 1.25 \right] $ was used. Fig.~\ref{fig:coverage_exp_setup} shows the constructed MAXCOVAR and the RANDCOVAR trees corresponding to the goal distribution $\mathcal{G}$ that were used for the experiments on coverage. The figure displays the first two states of the 6 DoF model: the x ($\operatorname{state}[0]$) and y ($\operatorname{state}[1]$) locations of the quadrotor on the x-y plane. Each node on the plot denotes the first two dimensions of the state distribution mean with edges between the nodes displayed as directed arrows such that the corresponding $N$-step control sequences are stored offline. The blue annulus around the two trees in Figure~\ref{fig:coverage_exp_setup} represents the region over which the query means were uniformly sampled. Both the trees in Fig.~\ref{fig:coverage_exp_setup} consist of 265 nodes and were generated with the same random seed.

The construction procedure for the MAXCOVAR tree was described in Algorithm~\ref{alg:construct_brt}. For any existing node $\nu_{k}$ on the tree that's selected to expand, and for any query mean $\mu_{q}$ sampled from a box around it, the new node covariance $\Sigma_{q, \mathrm{max}}$ and the corresponding edge controller $\mathscr{C}_{q, \mathrm{max}}$ is the result of solving the MAXCOVAR optimization problem.

For the construction of the RANDCOVAR tree, the node covariance $\Sigma_{\mathrm{rand}}$ was randomly sampled from the space of positive definite matrices, similar in spirit to \cite{csbrm_tro2024}, and the edge controller was given by the optimal steering control from $(\mu_{q}, \Sigma_{\mathrm{rand}})$ to $ (\mu_{k}, \Sigma_{k})$. To construct samples from the positive definite matrix space, the eigenvalues and orthonormal eigenvectors that constitute a positive definite matrix are sampled separately. It was observed empirically that randomly sampling node covariances resulted in rejecting a lot of candidate nodes due to the non-existence of a corresponding steering maneuver. Therefore, eigenvalues of the candidate node covariance $\Sigma_{\mathrm{rand}}$ were sampled to ensure that \ $ \lambda_{\mathrm{min}}( \Sigma_{\mathrm{rand}}) \leq \lambda_{\mathrm{min}}( \Sigma_{\mathrm{max}})$.

%The main loop of the BRT construction procedure was run for $n_{\mathrm{iter}} = 1500$ iterations and $539$ nodes were successfully added to the tree. The sampling procedure that was followed to select existing nodes to add to was uniform sampling over the state space exploiting Voronois bias for expansion of the tree. \cblue{Add construction details for the RANDCOVAR tree too.}

\begin{figure}[t]
    \centering
    \includegraphics[width=.9\columnwidth]{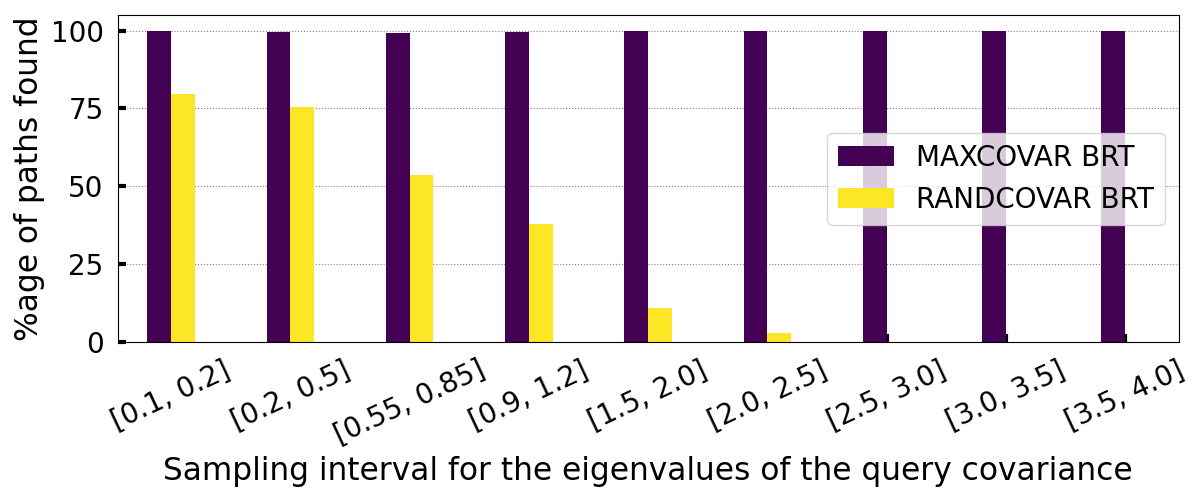}
    %\vspace*{-.25in} 
    \caption{Coverage experiment metrics}
    \label{fig:coverage_hist}
\end{figure}

\textit{\textbf{Coverage Experiment}}: The coverage experiment proceeds by sampling query distributions and attempting connects to the two trees to find paths. Query means were sampled from the blue annulus as shown in Fig.~\ref{fig:coverage_exp_setup}, and the query covariances were considered to be diagonal matrices where the diagonal entries were sampled uniformly over an interval. The experiment was repeated for different intervals of sampling the diagonal entries of the query covariance as shown in the x-axis of Fig.~\ref{fig:coverage_hist}. For each interval, the experiment was repeated 250 times and the percentage of times a path was found to the two trees was reported.

\textit{\textbf{Interpretation of the Coverage Experiment}}: From Fig.~\ref{fig:coverage_hist}, it can be seen that for intervals corresponding to higher candidate eigenvalues of the query covariance, the RANDCOVAR tree with randomly sampled node covariances is not able to find paths in contrast to MAXCOVAR tree where the node covariances and the edges were constructed explicitly to provide maximal coverage. This is a direct consequence of Lemma~\ref{lemma:sigma_brs_lemma} which says that for two positive definite matrices $\Sigma$ and $\Sigma^{-}$, there exist planning scenes such that distributions with a larger spectral radius of the covariance reach $\Sigma$ as compared to $\Sigma^{-}$ (Lemma~\ref{lemma:sigma_brs_lemma} follows a proof by construction, see Appendix~\ref{sec:appendix}).
%Since the nodes were added randomly in the RANDCOVAR tree s.t.\ the randomly sampled minimum eigenvalue of the randcovar tree node was smaller than the minimum eigenvalue returned by maxcovar (maximum possible value of the minimum eigenvalue satisfying finite-time reachability), the tree finds paths from distributions with overall a smaller spectra as compared to the tree built with the maxcovar edges.

\begin{figure}[htb]
    \centering
    \includegraphics[width=\columnwidth]{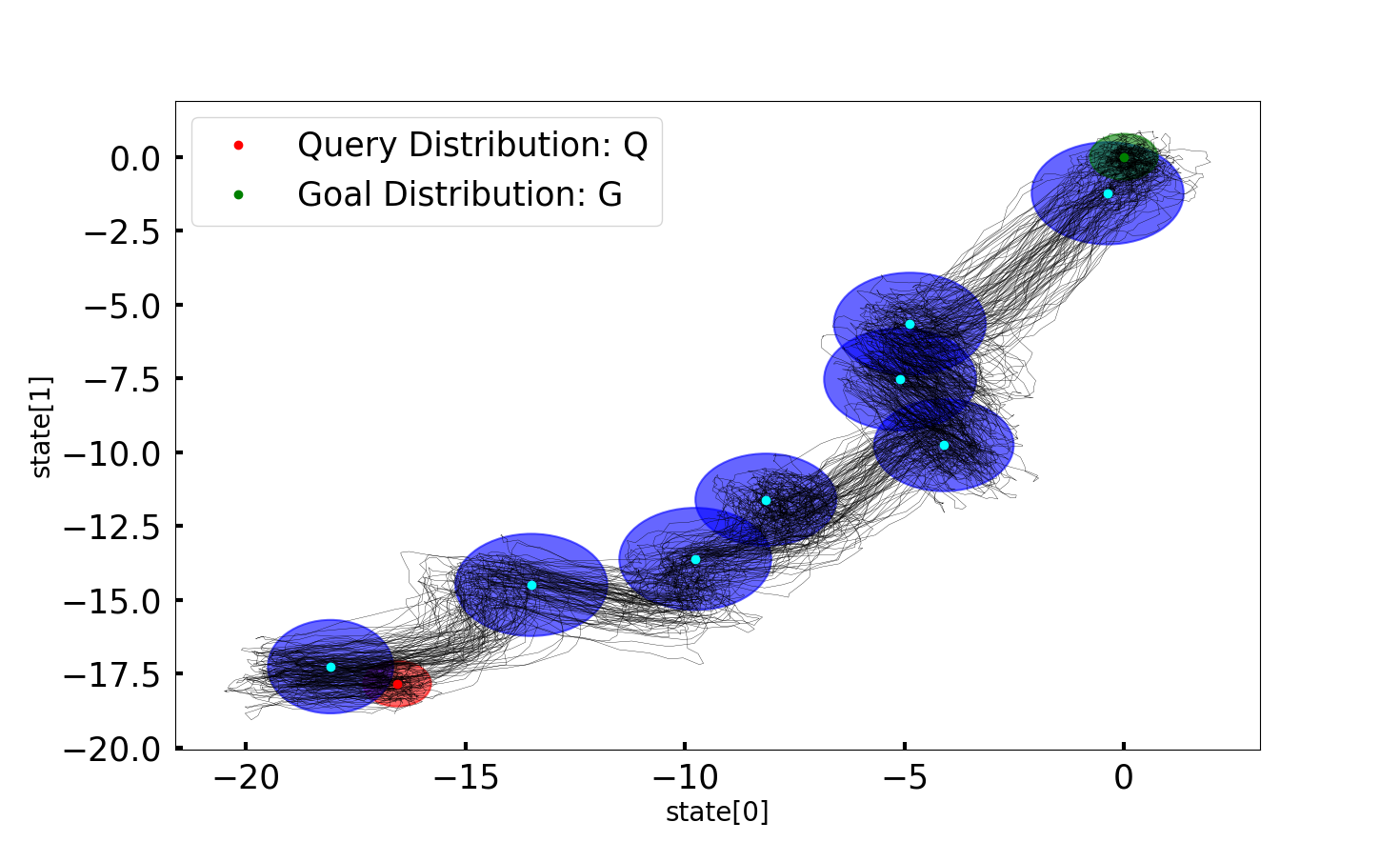}
    \vspace*{-.35in} 
    \caption{Real-time planning through the pre-computed MAXCOVAR BRT}
    \label{fig:9stepchain}
\end{figure}

\textit{\textbf{Real-time planning through the MAXCOVAR BRT}}: The generated tree and the associated control sequences were stored offline. For real-time planning, a query mean $\mu_{q}$ was randomly sampled in a box around the origin with a sampling radius of $[ \pm 25, \pm 25, \pm 25, \pm 25, 0, 0 ]$ with a query covariance of $\Sigma_{q} = 0.1 \ \mathbf{I}_{6\times6}$. Fig.~\ref{fig:9stepchain} shows one such run that resulted in a $9$-hop chained maneuver that drives the sampled query mean to the goal (a total of $9 \times N = 180$ timesteps). It took $1$ second to find a feasible path through the offline generated BRT, and around $2$ minutes to find a path through \ref{eq:J}($q, \mathcal{G}, N$) for $N = 180$. In Fig.~\ref{fig:9stepchain}, the red and green ellipses correspond to the $3$-sigma confidence interval of the sampled query node and the goal node respectively. The blue ellipses correspond to the distributions stored in the nodes of the discovered feasible path. Monte carlo simulations were performed and the system trajectories are represented by black lines. The only real-time computation that was done was to find a control sequence that drives the system from the red ellipse to the first blue ellipse in the sequence, the rest of the trajectories were propagated using controllers stored along each edge of the discovered feasible path of the BRT.

%It is important to note that the tree was constructed in the spirit of maximization of the volumes of regions of attraction so to speak (maximal covariance in our stochastic setting) to guarantee finding feasible paths from the largest possible set of initial distributions under control constraints. If the nodes and corresponding computed edge controllers were added to the tree keeping in mind a weighted combination of the volume of the region of attraction (size of the initial covariance for the edge) and performance, we would obtain a better trade-off.

\begin{comment}
    \begin{figure}[htb]
    \centering
    \includegraphics[width=0.5\textwidth]{paper-template-latex/imgs/exp_stats.png}
    \caption{Metrics of real-time planning through the BRT}
    \label{fig:metrics}
\end{figure}
\end{comment}

\begin{comment}
\begin{figure}[htbp]
    \centering
    \includegraphics[width=0.5\textwidth]{paper-template-latex/imgs/forest.png}
    \caption{Forest of BRTs}
    \label{fig:example}
\end{figure}

\begin{figure}[t]
    \centering
    \includegraphics[width=\columnwidth]{paper-template-latex/imgs/two_tree_compose.png}
    \caption{Composition of two BRTs}
    \label{fig:example}
\end{figure}
\end{comment}

%\section{Conclusion}
%\label{sec:conclusion}
%\section{Future Work}

%\section*{Acknowledgments}
\section*{Conclusion}
In this paper we introduced Maximal Covariance Backward Reachable Trees (MAXCOVAR BRT), a multi-query algorithm for probabilistic planning  under constraints on the control input with explicit coverage guarantees. Our contribution was a novel optimization formulation to add nodes and construct the corresponding edge controllers such that the generated roadmap results in provably maximal coverage. The notion of coverage of a roadmap was characterized formally via h-$\operatorname{BRS}$ (Backward Reachable Set of Distributions) of a tree of distributions, and our proposed method was supported by theoretical analysis as well as extensive simulation on a 6 DoF model.
%% Use plainnat to work nicely with natbib.
\small
\bibliographystyle{ieeetr}
\bibliography{references}

\appendix \label{sec:appendix}
\begin{comment}
    \begin{lemma}
    For $\Sigma_{\mathrm{max}}, \Sigma^{-}_{\mathrm{max}} \succ 0$, s.t.\ $\lambda_{\mathrm{min}}( \Sigma_{\mathrm{max}} ) > \lambda_{\mathrm{min}}( \Sigma^{-}_{\mathrm{max}} )$, $ h\text{-}\operatorname{BRS}( \mu, \Sigma_{\mathrm{max}} ) \supseteq h\text{-}\operatorname{BRS}( \mu, \Sigma^{-}_{\mathrm{max}} ) \ \forall \ h \geq 1 $ for all planning scenes. Also, there exist planning scenes $ \{ (\alpha_{x}, \beta_{x}, \epsilon_{x}), (\alpha_{u}, \beta_{u}, \epsilon_{u}) \} $ s.t.\ $ h\text{-}\operatorname{BRS}( \mu, \Sigma_{\mathrm{max}} ) \supset h\text{-}\operatorname{BRS}( \mu, \Sigma^{-}_{\mathrm{max}} ) \ \forall \ h \geq 1$.
\end{lemma}\label{lemma:sigma_brs_lemma_appendix}
\end{comment}

\begin{proof}[\textbf{Proof of Lemma}~\ref{lemma:sigma_brs_lemma}]
    We want to show that $ \ \forall \ p \in h\text{-}\operatorname{BRS}(\mu, \Sigma^{-}_{\mathrm{max}}); \ \exists \ \mathscr{C}_{p} $ s.t.\ $ p \xlongrightarrow[hN]{\mathscr{C}_{p}} (\mu, \Sigma_{\mathrm{max}}) $, what we refer to as the \textit{forward side} of the argument, and also that $\ \exists \ q $ s.t.\ $ q \in h\text{-}\operatorname{BRS}(\mu, \Sigma_{\mathrm{max}})$ and $ \nexists \ \mathscr{C}_{q} $ s.t.\ $ q \xlongrightarrow[hN]{\mathscr{C}_{q}} (\mu, \Sigma^{-}_{\mathrm{max}}) $, what we refer to as the \textit{backward side} of the argument.

    We first prove the forwards side of the argument. Let $p$ be any element of $h\text{-}\operatorname{BRS}(\mu, \Sigma^{-}_{\mathrm{max}})$. Therefore, $ \exists \ \mathscr{C}_p $ s.t.\ the equations (\ref{eq:covar_goal_maxcovar})-(\ref{eq:control_constraint_maxcovar}) are satisfied. From (\ref{eq:covar_goal_maxcovar}), $ \Sigma_{{hN}} (p, \mathscr{C}_{p}) $ s.t.\ $ \lambda_{\mathrm{max}}( \Sigma_{hN} (p, \mathscr{C}_{p}) ) \leq \lambda_{\mathrm{min}}(\Sigma^{-}_{\mathrm{max}})$. Since $ \lambda_{\mathrm{min}}( \Sigma^{-}_{\mathrm{max}} ) < \lambda_{\mathrm{min}}( \Sigma_{\mathrm{max}} ) $, $ \lambda_{\mathrm{max}}( \Sigma_{hN} (p, \mathscr{C}_{p}) ) \leq \lambda_{\mathrm{min}}(\Sigma_{\mathrm{max}})$. Therefore, $ p \in h\text{-}\operatorname{BRS}(\mu, \Sigma_{\mathrm{max}})$ under the same control sequence $ \mathscr{C}_{p} $ that drives $p$ to $ (\mu, \Sigma^{-}_{\mathrm{max}}) $. Since the above argument can be reproduced for any arbitrary distribution that reaches $(\mu, \Sigma^{-}_{\mathrm{max}})$, it follows that $ \forall \ p \in h\text{-}\operatorname{BRS}(\mu, \Sigma^{-}_{\mathrm{max}}), p \in h\text{-}\operatorname{BRS}(\mu, \Sigma_{\mathrm{max}})$.

    Now, we prove the backwards side of the argument. Specifically, we construct $q$ and a planning scene $ \{ (\alpha_{x}, \beta_{x}, \epsilon_{x}), (\alpha_{u}, \beta_{u}, \epsilon_{u}) \} $ such that $q \in h\text{-}\operatorname{BRS}(\mu, \Sigma_{\mathrm{max}})$ and $ q \notin h\text{-}\operatorname{BRS}(\mu, \Sigma^{-}_{\mathrm{max}}) $. Consider $q(\mu_{q}, \Sigma_{q})$ such that,
    \begin{align}
        \Sigma_{q}, \mathscr{C}_{q} \longleftarrow \operatorname{MAX-COVAR}( \mu_{q}, (\mu, \Sigma^{-}_{\mathrm{max}}), hN), \label{eq:sigma_q}
    \end{align}
    and $\lambda_{\mathrm{max}}(\Sigma_{q}) < \infty$. We also assume that for $ (\mu_{q}, \Sigma_{q})$ and $\mathscr{C}_{q}$, the corresponding state and control chance constraints (\ref{eq:state_constraint_maxcovar})-(\ref{eq:control_constraint_maxcovar}) are non-tight, i.e., are strict inequalities. This is a mild assumption and can be shown to hold by constructing $\mu_{q}$ in an appropriate manner.

    From (\ref{eq:sigma_q}), $\nexists \ \Sigma_{q^{+}}, \mathscr{C}_{q^{+}}$ s.t.\ $ \lambda_{\mathrm{min}}(\Sigma_{q^{+}}) > \lambda_{\mathrm{min}}(\Sigma_q)$ and 
    $ (\mu_q, \Sigma_{q^{+}}) \xlongrightarrow[kN]{\mathscr{C}_{q^{+}}} (\mu, \Sigma^{-}_{\mathrm{max}})$ $\ast$\label{sent:perturb}.
    
    Now, consider a perturbation to $q$ such that $\Sigma(\epsilon) = \Sigma_{q} + \epsilon \mathrm{I}$ for some $\epsilon > 0$. From $(\ast)$~\ref{sent:perturb}, $ (\mu_{q}, \Sigma(\epsilon)) \notin h\text{-}\operatorname{BRS}(\mu, \Sigma^{-}_{\mathrm{max}})$ since $ \lambda_{\mathrm{min}}(\Sigma(\epsilon)) > \lambda_{\mathrm{min}}(\Sigma_{q})$ for $\epsilon > 0$. We show that $\exists$ planning scenes $\{ (\alpha_{x}, \beta_{x}, \epsilon_{x}), (\alpha_{u}, \beta_{u}, \epsilon_{u}) \}$ $\exists \ \epsilon > 0$ s.t.\ $ (\mu_q, \Sigma(\epsilon)) \in h\text{-}\operatorname{BRS}(\mu, \Sigma_{\mathrm{max}})$.

    Consider $\mathscr{C}(\lambda)$ as the candidate control sequence s.t.\ $ (\mu_{q}, \Sigma(\epsilon)) \xlongrightarrow[hN]{\mathscr{C}(\lambda)} (\mu, \Sigma_{\mathrm{max}})$ where $\lambda \coloneqq \{\lambda_{k}\}_{t=0}^{hN-1}$ is a perturbation to $\mathscr{C}_{q}$ defined as follows,
    \begin{align*}
        K_{k}(\mathscr{C}(\lambda)) &= (1-\lambda_{k}) K_{k}(\mathscr{C}_{q}), \\
        v_{k}(\mathscr{C}(\lambda)) &= v_{k}(\mathscr{C}_{q}),
    \end{align*}
    $\forall \ k = 0, 1, \cdots hN-1.$ Under the perturbed control sequence $\mathscr{C}(\lambda)$, the mean dynamics are unaffected since the feedforward term $v_k(\mathscr{\lambda}) = v_{k}(\mathscr{C}_{q})$ is unperturbed, i.e.\ $\mu_{k}(\mathscr{C}(\lambda)) = \mu_{k}(\mathscr{C}_{q})$. We now express the dynamics of the state covariance under the perturbed initial covariance and control sequence $\Sigma(\epsilon)$ and $\mathscr{C}(\lambda)$ as $ \Sigma_{k}(\Sigma(\epsilon), \mathscr{C}(\lambda))$. We have,
    \begin{align*}
        \Sigma_{1}(\Sigma, \mathscr{C}) &= (A + BK_{0}(\mathscr{C}))(\Sigma_{q} + \epsilon \mathrm{I}) (A + BK_{0}(\mathscr{C}))\t
        + DD\t  \\
        &= \Sigma_{1}(\Sigma_{q}, \mathscr{C}) + \epsilon \gamma_{1}(\mathscr{C})
    \end{align*}
    where $ \Sigma_{1}(\Sigma_{q}, \mathscr{C}) = (A + BK_{0}(\mathscr{C}))\Sigma_{q}(A + BK_{0}(\mathscr{C}))\t $ and $ \gamma_{1}(\mathscr{C}) = (A + BK_{0}(\mathscr{C}))(A + BK_{0}(\mathscr{C}))\t$. Through induction, we can express $ \Sigma_{k}(\Sigma, \mathscr{C})$ as,
    \begin{align}
        \Sigma_{k}(\Sigma, \mathscr{C}) = \Sigma_{k}(\Sigma_{q}, \mathscr{C}) + \epsilon \gamma_{k}(\mathscr{C}), \label{eq:sigmak_recurrence}
    \end{align}
    where $ \Sigma_{0}(\Sigma_{q}, \mathscr{C}) = \Sigma_{q}$, $ \gamma_{k}(\mathscr{C}) = (A + BK_{k-1}(\mathscr{C}))\gamma_{k-1}(\mathscr{C})(A + BK_{k-1}(\mathscr{C}))\t$, and $ \gamma_{0}(\mathscr{C}) = \mathrm{I}$.

    Writing the state chance constraints (\ref{eq:state_constraint_maxcovar}) corresponding to $ \Sigma(\epsilon) $ and $ K(\mathscr{C}(\lambda))$,
    \begin{align}
        \mathscr{S}_{k}(\Sigma, \mathscr{C}) &= \Phi^{-1}(1-\epsilon_{x})\sqrt{ \alpha_{x}\t \Sigma_{k}(\Sigma, \mathscr{C})\alpha_{x}} + \alpha_{x}\t \mu_{k}(\mathscr{C} ) - \beta_{x}. %\label{eq:state_chanceconstraint_perturb}
    \end{align}
    From (\ref{eq:sigmak_recurrence}) and (\ref{eq:state_chanceconstraint_perturb}), and using $\sqrt{a + b} \leq \sqrt{a} + \sqrt{b} \ \forall \ a,b \geq 0$,
    \begin{align}
        \mathscr{S}_{k}(\Sigma, \mathscr{C}) &\leq  \Phi^{-1}(1-\epsilon_{x}) \left[\sqrt{ \alpha_{x}\t \Sigma_{k}(\Sigma_{q}, \mathscr{C})\alpha_{x}} + \sqrt{\epsilon}\sqrt{ \alpha_{x}\t \gamma_{k}(\mathscr{C})\alpha_{x}}  \right] \nonumber \\ 
        & + \alpha_{x}\t \mu_{k}(\mathscr{C}) - \beta_{x}. \label{eq:state_chanceconstraint_perturb}
    \end{align}
    From the state chance constraint corresponding to $\Sigma_{q}$ and $\mathscr{C}_{q}$,
    \begin{align}
        \mathscr{S}_{k}(\Sigma_{q}, \mathscr{C}_{q}) = \Phi^{-1}(1-\epsilon_{x})\sqrt{ \alpha_{x}\t \Sigma_{k}(\Sigma_{q}, \mathscr{C}_{q})\alpha_{x}} + \alpha_{x}\t \mu_{k}(\mathscr{C}_{q}) - \beta_{x} < 0. \nonumber
    \end{align}
    Therefore,
    \begin{align}
        \alpha_{x}\t \mu_{k}(\mathscr{C}_{q}) - \beta_{x} = -z_{s,k} - \left[ \Phi^{-1}(1-\epsilon_{x})\sqrt{ \alpha_{x}\t \Sigma_{k}(\Sigma_{q}, \mathscr{C}_{q})\alpha_{x}} \right], \label{eq:state_constraint_slack_substitute}
    \end{align}
    where $z_{s,k} > 0$ is the slack variable associated with the state chance constraint at the $k$-th time-step for the system initialized at $\Sigma_{q}$ and evolving under $\mathscr{C}_{q}$. Substituting (\ref{eq:state_constraint_slack_substitute}) in (\ref{eq:state_chanceconstraint_perturb}), and since $ \mu_{k}(\mathscr{C}) = \mu_{k}(\mathscr{C}_{q})$,
    \begin{align}
        \mathscr{S}_{k}(\Sigma, \mathscr{C}) &\leq  \Phi^{-1}(1-\epsilon_{x}) \left[\sqrt{ \alpha_{x}\t \Sigma_{k}(\Sigma_{q}, \mathscr{C})\alpha_{x}} -\sqrt{ \alpha_{x}\t \Sigma_{k}(\Sigma_{q}, \mathscr{C}_{q})\alpha_{x}}  \right] \nonumber \\ 
        & + \sqrt{\epsilon}\sqrt{ \alpha_{x}\t \gamma_{k}(\mathscr{C})\alpha_{x}} - z_{s,k}. \label{eq:state_chanceconstraint_perturb_1}        
    \end{align}
    We construct perturbations $\epsilon$ and $\lambda$ s.t.\ $ \mathscr{S}_{k}(\Sigma, \mathscr{C}) \leq 0 \ \forall \ k = 0, 1, \cdots, hN-1$. The following holds for the desired value of the perturbations,
    \begin{align}
        \sqrt{\epsilon}\sqrt{ \alpha_{x}\t \gamma_{k}(\mathscr{C})\alpha_{x}} \leq z_{s,k} &- \Phi^{-1}(1-\epsilon_{x}) \left[ \sqrt{ \alpha_{x}\t \Sigma_{k}(\Sigma_{q}, \mathscr{C})\alpha_{x}} \right. \nonumber \\  &- \left. \sqrt{ \alpha_{x}\t \Sigma_{k}(\Sigma_{q}, \mathscr{C}_{q})\alpha_{x}} \right] \tag{S}\label{eq:state_constraint_ineq}
    \end{align}
    Following a similar analysis as the state chance constraints above, we can express the control chance constraints (\ref{eq:control_constraint_maxcovar}) corresponding to $\Sigma(\epsilon)$ and $\mathscr{C}(\lambda)$ as,
    \begin{align}
        \mathcal{U}_{k}(\Sigma, \mathscr{C}) &\leq  \Phi^{-1}(1-\epsilon_{u}) \left[\sqrt{ \alpha_{u}\t K_k(\mathscr{C}) \Sigma_{k}(\Sigma_{q}, \mathscr{C})K_k\t(\mathscr{C})\alpha_{u}} \right. \nonumber \\
        & \left. -\sqrt{ \alpha_{u}\t K_k(\mathscr{C})\Sigma_{k}(\Sigma_{q}, \mathscr{C}_{q})K_k\t(\mathscr{C})\alpha_{u}}  \right] \nonumber \\ &+ \sqrt{\epsilon}\sqrt{ \alpha_{u}\t K_k(\mathscr{C}) \gamma_{k}(\mathscr{C}) K_k\t(\mathscr{C}) \alpha_{u}} - z_{u,k}, \label{eq:control_chanceconstraint_perturb}
    \end{align}
    where $z_{u,k} > 0$ is the slack variable associated with the control chance constraint at the $k$-th time-step. For the desired values of the perturbations, $ \mathcal{U}_{k}(\Sigma, \mathscr{C}) \leq 0$, and we get the following,
    \begin{align}
        \sqrt{\epsilon}&\sqrt{ \alpha_{u}\t K_k(\mathscr{C}) \gamma_{k}(\mathscr{C}) K_k\t(\mathscr{C}) \alpha_{u}} \leq  \nonumber \\ & \Phi^{-1}(1-\epsilon_{u}) \left[ \sqrt{ \alpha_{u}\t K_k(\mathscr{C})\Sigma_{k}(\Sigma_{q}, \mathscr{C}_{q})K_k\t(\mathscr{C})\alpha_{u}} \right. \nonumber \\ &- \left. \sqrt{ \alpha_{u}\t K_k(\mathscr{C}) \Sigma_{k}(\Sigma_{q}, \mathscr{C})K_k\t(\mathscr{C})\alpha_{u}}  \right]  + z_{u,k}. \tag{U} \label{eq:control_constraint_ineq}
    \end{align}
    For there to always exist some $\epsilon > 0$ s.t.\ (\ref{eq:state_constraint_ineq}) and (\ref{eq:control_constraint_ineq}) hold true, it is a sufficient condition that $ \sqrt{ \alpha_{x}\t \Sigma_{k}(\Sigma_{q}, \mathscr{C}_{q})\alpha_{x}} - \sqrt{ \alpha_{x}\t \Sigma_{k}(\Sigma_{q}, \mathscr{C}(\lambda))\alpha_{x}} \geq 0$, and $ \sqrt{ \alpha_{u}\t K_{k}(\mathscr{C}) \Sigma_{k}(\Sigma_{q}, \mathscr{C}_{q}) K\t_{k}(\mathscr{C}) \alpha_{u}} - \sqrt{ \alpha_{u}\t K_{k} \Sigma_{k}(\Sigma_{q}, \mathscr{C}(\lambda)) K\t_{k} \alpha_{u}} \geq 0$, since $ z_{s,k}$ and $z_{u,k}$ are strictly positive $ \forall \ k = 0, 1, \cdots, hN-1$. Therefore, we construct $\lambda$ such that $ \Sigma_{k}(\Sigma_{q}, \mathscr{C}(\lambda)) \prec \Sigma_{k}(\Sigma_{q}, \mathscr{C}_{q}) \ \forall \ k=0, 1, \cdots, hN-1 $. Writing $\Sigma_{1}(\Sigma_{q}, \mathscr{C}(\lambda))$ in terms of $\Sigma_{1}(\Sigma_{q},
    \mathscr{C}_{q})$, \mycomment{We construct a perturbation $\lambda$ to $\mathscr{C}_{q}$ such that $ \Sigma_{k} (\Sigma_{q}, \mathscr{C}(\lambda)) \prec \Sigma_{k}(\Sigma_{q}, \mathscr{C}_{q})$, and $ \gamma_{k}(\mathscr{C}(\lambda)) \prec \gamma_{k}(\mathscr{C}_{q}) \ \forall \ k$.}
    \begin{align}
        \Sigma_{1}(\Sigma_{q}, \mathscr{C}(\lambda)) &= (A + BK_{0}(\lambda)) \Sigma_{q} (A + BK_{0}(\lambda))\t + DD\t \nonumber \\
        &=\begin{aligned}
        \Sigma_{1}(&\Sigma_{q}, \mathscr{C}_{q}) + \lambda^{2}_{0}BK_{0}\Sigma_{q}(BK_{0})\t  \\ &\qquad - \lambda_{0} [ BK_{0} \Sigma_{q} (A + BK_{0})\t \\ &\qquad \qquad + (A+BK_{0})\Sigma_{q} (BK_{0})\t ].
        \end{aligned} \label{eq:sigma1_c_cq}
    \end{align}
    From (\ref{eq:sigma1_c_cq}), $\lambda^{2}_{0}BK_{0}\Sigma_{q}(BK_{0})\t \prec \lambda_{0} [ BK_{0} \Sigma_{q} (A + BK_{0})\t + (A+BK_{0})\Sigma_{q} (BK_{0})\t ]$ is a sufficient condition to ensure $ \Sigma_{1} (\Sigma_{q}, \mathscr{C}(\lambda)) \prec \Sigma_{1}(\Sigma_{q}, \mathscr{C}_{q})$. This is done by constructing $\lambda_{0}$ s.t.\ $ \lambda_{\mathrm{max}}(\lambda^{2}_{0}BK_{0}\Sigma_{q}(BK_{0})\t) < \lambda_{\mathrm{min}}(\lambda_{0} [ BK_{0} \Sigma_{q} (A + BK_{0})\t + (A+BK_{0})\Sigma_{q} (BK_{0})\t) ]$. Therefore,
    \begin{align}
        \lambda_{0} < \frac{\lambda_{\mathrm{min}}(BK_{0} \Sigma_{q} (A + BK_{0})\t + (A+BK_{0})\Sigma_{q} (BK_{0})\t)}{ \lambda_{\mathrm{max}}(BK_{0}\Sigma_{q}(BK_{0})\t) }. \label{eq:lambda_0}
    \end{align}
    Repeating the above arguments recursively for $\Sigma_{k}(\Sigma_{q}, \mathscr{C}(\lambda))$ and $\Sigma_{k}(\Sigma_{q}, \mathscr{C}_{q})$, $\Sigma_{k+1}(\Sigma_{q}, \mathscr{C}(\lambda)) = (A + BK_{k}(\lambda)) \Sigma_{k}(\Sigma_{q}, \mathscr{C}(\lambda)) (A + BK_{k}(\lambda))\t + DD\t$. Assuming $\lambda_{k-}$ is such that $ \Sigma_{k}(\Sigma_{q}, \mathscr{C}(\lambda)) \prec \Sigma_{k}(\Sigma_{q}, \mathscr{C}_{q}) $, $\lambda^{2}_{k}BK_{k}\Sigma_{k}(\Sigma_{q}, \mathscr{C}_{q})(BK_{k})\t \preceq \lambda_{k} [ BK_{k} \Sigma_{k}(\Sigma_{q}, \mathscr{C}_{q}) (A + BK_{k})\t + (A+BK_{k})\Sigma_{k}(\Sigma_{q}, \mathscr{C}_{q}) (BK_{k})\t ]$ is a sufficient condition to ensure $ \Sigma_{k+1}(\Sigma_{q}, \mathscr{C}(\lambda)) \preceq \Sigma_{k+1}(\Sigma_{q}, \mathscr{C}_{q})$. Therefore,
    \begin{align}
        \lambda_{k} <= \frac{\mathscr{N}^{\Sigma}_{k}}{ \lambda_{\mathrm{max}}(BK_{k}\Sigma_{k}(\Sigma_{q}, \mathscr{C}_{q})(BK_{k})\t) } \label{eq:lambda_k}
        %\lambda_{k} <= \frac{\lambda_{\mathrm{min}}(BK_{k} \Sigma_{k}(\Sigma_{q}, \mathscr{C}_{q}) (A + BK_{k})\t + (A+BK_{k})\Sigma_{k}(\Sigma_{q}, \mathscr{C}_{q}) (BK_{k})\t)}{ \lambda_{\mathrm{max}}(BK_{k}\Sigma_{k}(\Sigma_{q}, \mathscr{C}_{q})(BK_{k})\t) },
    \end{align}
    $\forall \ k=0, 1, \cdots, hN-1$, where $ \mathscr{N}^{\Sigma}_{k} \coloneqq \lambda_{\mathrm{min}}(BK_{k} \Sigma_{k}(\Sigma_{q}, \mathscr{C}_{q}) (A + BK_{k})\t + (A+BK_{k})\Sigma_{k}(\Sigma_{q}, \mathscr{C}_{q}) (BK_{k})\t) $.

    We define quantities $\epsilon_{s,k}(\alpha_{x}, \epsilon_{x}; \lambda)$, and $\epsilon_{u,k}(\alpha_{u}, \epsilon_{u}; \lambda)$ as follows,
    \begin{align}
        &\epsilon_{s,k}(\alpha_{x}, \epsilon_{x}; \lambda) \coloneqq \frac{z_{s,k}}{\lVert \alpha_{x} \rVert \sqrt{ \lambda_{\mathrm{max}}(\gamma_{k}(\mathscr{C}(\lambda))) } } \nonumber \\ &+ \Phi^{-1}(1-\epsilon_{x}) \frac{ \sqrt{\lambda_{\mathrm{min}}( \Sigma_{k}(\Sigma_{q}, \mathscr{C}_{q}))} - \sqrt{\lambda_{\mathrm{max}}( \Sigma_{k}(\Sigma_{q}, \mathscr{C}(\lambda)))} }{\sqrt{ \lambda_{\mathrm{max}}(\gamma_{k}(\mathscr{C}(\lambda))) }}, \label{eq:epsilon_sk_lb} \\
        &\epsilon_{u,k}(\alpha_{u}, \epsilon_{u}; \lambda) \coloneqq \frac{z_{u,k}}{\lvert (1 - \lambda_{k}) \rvert\lVert K\t_{k}(\mathscr{C}_{q})\alpha_{u} \rVert \sqrt{ \lambda_{\mathrm{max}}(\gamma_{k}(\mathscr{C}(\lambda))) } } \nonumber \\ &+ \Phi^{-1}(1-\epsilon_{u}) \frac{ \sqrt{\lambda_{\mathrm{min}}( \Sigma_{k}(\Sigma_{q}, \mathscr{C}_{q}))} - \sqrt{\lambda_{\mathrm{max}}( \Sigma_{k}(\Sigma_{q}, \mathscr{C}(\lambda)))} }{\sqrt{ \lambda_{\mathrm{max}}(\gamma_{k}(\mathscr{C}(\lambda))) }}. \label{eq:epsilon_uk_lb}
    \end{align}
    $\epsilon_{s,k}(\alpha_{x}, \epsilon_{x}; \lambda)$, and $\epsilon_{u,k}(\alpha_{u}, \epsilon_{u}; \lambda)$ are lower bounds on the RHS of (\ref{eq:state_constraint_ineq}) and (\ref{eq:control_constraint_ineq}) respectively. Therefore, a sufficient condition for (\ref{eq:state_constraint_ineq}) and (\ref{eq:control_constraint_ineq}) to hold is the following,
    \begin{align}
        \epsilon \leq \epsilon_{s}(\alpha_{x}, \epsilon_{x}; \lambda) \coloneqq \min_{k} \epsilon_{s,k}(\alpha_{x}, \epsilon_{x}; \lambda), \\
        \epsilon \leq \epsilon_{u}(\alpha_{u}, \epsilon_{u}; \lambda) \coloneqq \min_{k} \epsilon_{u,k}(\alpha_{u}, \epsilon_{u}; \lambda).
    \end{align}
    For specified $\lambda$, $ \epsilon_{s,k}(\alpha_{x}, \epsilon_{x}; \lambda)$ and $ \epsilon_{u,k}(\alpha_{u}, \epsilon_{u}; \lambda) $ are functions of the planning scene $ \{ (\alpha_{x}, \epsilon_{x}), (\alpha_{u}, \epsilon_{u}) \} $ s.t.\ $ \epsilon_{s,k}(\alpha_{x}, \epsilon_{x}; \lambda)$ and $ \epsilon_{u,k}(\alpha_{u}, \epsilon_{u}; \lambda) $ are monotonically decreasing functions of $ \lVert \alpha_{x} \rVert, \lVert \alpha_{u} \rVert \neq 0 $ respectively.
    
    From $ \lambda$ specified earlier, the above inequalities are a function of just the planning scene $ \{ (\alpha_{x}, \epsilon_{x}), (\alpha_{u}, \epsilon_{u}) \}$. We now derive a condition on $\epsilon$ from the terminal goal reaching constraint, and show that there always exists some $\epsilon >0$ such that all the conditions hold by choosing the planning scene in a careful manner. Writing the desired terminal covariance constraint under the perturbations,
    \begin{align}
        \lambda_{\mathrm{min}}( \Sigma^{-}_{\mathrm{max}} ) < \lambda_{\mathrm{max}}( \Sigma_{hN} (\Sigma, \mathscr{C})) \leq \lambda_{\mathrm{min}}(\Sigma_{\mathrm{max}}).
    \end{align}
    We define the function $f: [0,\infty) \rightarrow \mathbb{R}^{+}$ as the following,
    \begin{align}
        f(\epsilon; \lambda) \coloneqq \lambda_{\mathrm{max}}(\Sigma_{hN} (\Sigma, \mathscr{C})) = \lambda_{\mathrm{max}}(\Sigma_{hN} (\Sigma_{q}, \mathscr{C}) + \epsilon \gamma_{k}(\mathscr{C})),
    \end{align}
    where for a specified $\lambda$, $f(.; \lambda)$ is a function of $\epsilon$ alone. In particular, $f(.)$ is a continuous, strictly increasing (under mild assumptions on $\mathscr{C}_{q}$ i.e.\ $(A + BK_{k})$ is full rank $\forall \ k$), and a convex function of $\epsilon$. $f(0; \lambda) = \lambda_{\mathrm{max}}( \Sigma_{hN}( \Sigma_{q}, \mathscr{C})) \leq \lambda_{\mathrm{max}}( \Sigma_{hN}( \Sigma_{q}, \mathscr{C}_{q}))$. Also from (\ref{eq:covar_goal_maxcovar}), $ \lambda_{\mathrm{max}} ( \Sigma_{hN} (\Sigma_{q}, \mathscr{C}_{q})) \leq \lambda_{\mathrm{min}}(\Sigma^{-}_{\mathrm{max}})$. Therefore, $ \exists \ \epsilon^{+}$ s.t.\ $ f(\epsilon^{+}; \lambda) =  \lambda_{\mathrm{min}}(\Sigma^{-}_{\mathrm{max}})$ s.t.\ $ f(\epsilon; \lambda) > \lambda_{\mathrm{min}}(\Sigma^{-}_{\mathrm{max}}) \ \forall \ \epsilon > \epsilon^{+}$.

    From the specified $\lambda$ as before (\ref{eq:lambda_0}) (\ref{eq:lambda_k}), and the existence of $\epsilon^{+}$ (note that the existence of such an $\epsilon^{+}$ and it's value is independent of the choice of the planning scene parameters), we choose $ \{ (\alpha_{x}, \epsilon_{x}), (\alpha_{u}, \epsilon_{u}) \} $ s.t.\ $\epsilon_{s,k}(\alpha_{x}, \epsilon_{x}; \lambda), \epsilon_{u,k}(\alpha_{u}, \epsilon_{u}; \lambda) > \epsilon^{+} \ \forall \ k$.
    
    For such a choice of $\lambda$, $ \{ (\alpha_{x}, \epsilon_{x}), (\alpha_{u}, \epsilon_{u}) \}$, and $ \epsilon \in (\epsilon^{+}, \operatorname{min}(\epsilon_{s}(\alpha_x, \epsilon_x; \lambda), \epsilon_{u}(\alpha_u, \epsilon_u; \lambda)) $,  $ (\mu_{q}, \Sigma(\epsilon)) \xlongrightarrow[kN]{\mathscr{C}(\lambda)} (\mu_{\mathrm{cand}}, \Sigma_{\mathrm{max}})$ and $ \nexists \ \mathscr{C}_{q^{+}} $ s.t.\ $ (\mu_q, \Sigma(\epsilon)) \xlongrightarrow[hN]{\mathscr{C}_{q^{+}}} (\mu_{\mathrm{cand}}, \Sigma^{-}_{\mathrm{max}})$. Hence, proved.
    
\end{proof}
\begin{comment}
    \begin{theorem}[Coverage]\label{theorem:coverage_appendix}
    $\text{h-}\operatorname{BRS}(\mathcal{T}^{(r)}_{\mathrm{MAXCOVAR}} ) \supseteq \text{h-}\operatorname{BRS}(\mathcal{T}^{(r)}_{\mathrm{ANY}} )$ for all planning scenes, and there always exist planning scenes such that $\text{h-}\operatorname{BRS}(\mathcal{T}^{(r)}_{\mathrm{MAXCOVAR}} ) \supset \text{h-}\operatorname{BRS}(\mathcal{T}^{(r)}_{\mathrm{ANY}} ) \ \forall \ r \geq 1, \ \forall \ h \geq 1$.
\end{theorem}
\end{comment}

\begin{proof}[\textbf{Proof of Theorem}~\ref{theorem:coverage}]
    We define the h-$\operatorname{BRS}$ of a distribution $p$ as the following: 
    h-$\operatorname{BRS}(p) =  \{ (\mu_{q}, \Sigma_{q}) \vert \operatorname{FEASIBLE}(q, p, hN) = \operatorname{TRUE} \}$.
    Using the above definition, the h-$\operatorname{BRS}(\mathcal{T})$ of a tree of distributions $\mathcal{T}$ is defined as,
    \begin{align}
        \text{h-}\operatorname{BRS} = \bigcup_{i \in \nu(\mathcal{T})} \text{($h - d_{i}$)-}\operatorname{BRS}(\nu_{i}) \label{eq:tree_brs_def}
    \end{align}
    where $ d_{i} $ is the distance of the $i$-th node from the goal node in terms of the number of hops. Let $ \mathcal{T}^{(n)} $ be the state of the tree at the end of $n$ iterations of the $\operatorname{ADD}$ procedure. We first show that $ \exists \ \{ ( \alpha_{x}, \beta_{x}, \epsilon_{x} ), ( \alpha_{u}, \beta_{u}, \epsilon_{u} ) \}$ s.t.\ h-$\operatorname{BRS}( \mathcal{T}^{(n+1)} _{\mathrm{MAXCOVAR}}) \supset \text{h-}\operatorname{BRS}( \mathcal{T}^{(n+1)} _{\mathrm{ANY}})$ where $ \mathcal{T}^{(n+1)}_{\mathrm{MAXCOVAR}}$, $ \mathcal{T}^{(n+1)}_{\mathrm{ANY}}$ are trees obtained from a common $\mathcal{T}^{(n)}$ by following one more iteration of the $\operatorname{ADD}$ procedure with the edge constructed according to the $\operatorname{MAXCOVAR}$ algorithm that explicitly maximizes the minimum eigenvalue of the initial covariance  versus any other algorithm for edge construction $\operatorname{ANY}$. The $\operatorname{ANY}$ procedure could for instance sample belief nodes randomly, or add edges that optimize for a performance regularized objective of the minimum eigenvalue of the initial covariance matrix.

    The $(n+1)$-th iteration proceeds by attempting to add a sampled mean $\mu_{q}$ to an existing node $ \nu^{(n)} $ on the tree, where $\nu^{(n)}$ is the node sampled on the $(n+1)$-th iteration. We have,
    \begin{align}
        \Sigma_{\mathrm{max}}, \mathcal{C}_{\mathrm{max}} \longleftarrow \operatorname{MAXCOVAR}(\mu_{q}, ( \mu^{(n)}, \Sigma^{(n)}), N),
    \end{align}
    and,
    \begin{align}
        \Sigma_{\mathrm{any}}, \mathcal{C}_{\mathrm{any}} \longleftarrow \operatorname{ANY}(\mu_{q}, ( \mu^{(n)}, \Sigma^{(n)}), N),
    \end{align}
    s.t.\ $ \lambda_{\mathrm{min}}(\Sigma_{\mathrm{max}}) > \lambda_{\mathrm{min}}(\Sigma_{\mathrm{any}})$. Consider the two nodes added at the $(n+1)$-th iteration: $ \nu^{(n+1)}_{\mathrm{MAXCOVAR}} = (\mu_{q}, \Sigma_{\mathrm{max}})$; and $ \nu^{(n+1)}_{\mathrm{ANY}} = (\mu_{q}, \Sigma_{\mathrm{any}})$ with distance from the goal node as $ d^{(n+1)} $. From Lemma~\ref{lemma:sigma_brs_lemma}, $ \exists \ \{ ( \alpha_{x}, \beta_{x}, \epsilon_{x} ), ( \alpha_{u}, \beta_{u}, \epsilon_{u} ) \}$ s.t.\  $ \text{t-}\operatorname{BRS}(\nu^{(n+1)}_{\mathrm{MAXCOVAR}}) \supset \text{t-}\operatorname{BRS}(\nu^{(n+1)}_{\mathrm{ANY}}) \ \forall \ t \geq 1$. Therefore,
    \begin{align}
        \text{$(h-d^{(n+1)})$-}\operatorname{BRS}(\nu^{(n+1)}_{\mathrm{MAXCOVAR}}) \supset \text{$(h-d^{(n+1)})$-}\operatorname{BRS}(\nu^{(n+1)}_{\mathrm{ANY}}), \label{eq:hbrs_nplus1}
    \end{align}
    $ \forall \ h > d^{(n+1)} $. From (\ref{eq:tree_brs_def}),
    \begin{align}
        \text{h-}\operatorname{BRS}(\mathcal{T}^{(n+1)}_{\mathrm{MAXCOVAR}}) &= \text{h-}\operatorname{BRS}(\mathcal{T}^{(n)}) \bigcup \nonumber \\ &(h - d^{(n+1)})-\operatorname{BRS}( \nu^{(n+1)}_\mathrm{MAXCOVAR}),
    \end{align}
    and,
    \begin{align}
        \text{h-}\operatorname{BRS}(\mathcal{T}^{(n+1)}_{\mathrm{ANY}}) &= \text{h-}\operatorname{BRS}(\mathcal{T}^{(n)}) \bigcup \nonumber \\ &(h - d^{(n+1)})-\operatorname{BRS}( \nu^{(n+1)}_\mathrm{ANY}).
    \end{align}
    From (\ref{eq:hbrs_nplus1}), it follows that there exist planning scenes such that $ \text{h-}\operatorname{BRS}(\mathcal{T}^{(n+1)}_{\mathrm{MAXCOVAR}}) \supset \text{h-}\operatorname{BRS}(\mathcal{T}^{(n+1)}_{\mathrm{ANY}})$ when starting from a common tree state $ \mathcal{T}^{(n)} $. Considering the common tree state as just a singleton set of the goal distribution, $ \mathcal{T}^{(0)} = \{\mathcal{G}\}$. Therefore, there exist planning scenes such that $\text{h-}\operatorname{BRS}(\mathcal{T}^{(1)}_{\mathrm{MAXCOVAR}}) \supset \text{h-}\operatorname{BRS}(\mathcal{T}^{(1)}_{\mathrm{ANY}})$. The result for a general $r$ follows from the above. Hence, proved.
    
\end{proof}

\end{document}